%% file: main.tex
\documentclass[11pt]{article}
\usepackage[left=1in, right=1in, top=1in,
bottom=1in]{geometry}

\usepackage{amsmath}
\usepackage{amsthm}

\usepackage{amssymb}
\usepackage{amsfonts}
\usepackage{bbm}
\usepackage{bm}
\usepackage{graphicx}
\usepackage{grffile}
\usepackage{wrapfig,epsfig}
\usepackage{url}
\usepackage{epstopdf}
\usepackage{enumitem}
\usepackage{booktabs}
\usepackage{multirow}
\usepackage{makecell}
\usepackage{subcaption}

\usepackage{algorithm}
\usepackage{algorithmicx}
\usepackage{algpseudocode}

\usepackage{placeins}

\usepackage{tikz}
\usetikzlibrary{spy,calc}

\usepackage{adjustbox}
\usepackage{tabularx}
\usepackage{siunitx}

\usepackage{xcolor}
\usepackage{tcolorbox}
\usepackage{threeparttable}
\usepackage{array}
\usepackage{colortbl}

\usepackage{wrapfig}
\usepackage{pgfplots}
\pgfplotsset{compat=1.18}

\usepackage{newpxtext}
\usepackage{newpxmath}

\usepackage[pagebackref,breaklinks,colorlinks,citecolor=blue]{hyperref}

\theoremstyle{plain}
\newtheorem{theorem}{Theorem}[section]
\newtheorem{lemma}[theorem]{Lemma}

\newtheorem{corollary}[theorem]{Corollary}

\theoremstyle{definition}
\newtheorem{definition}[theorem]{Definition}
\newtheorem{assumption}[theorem]{Assumption}

\newtheorem{remark}[theorem]{Remark}

\renewcommand{\textsc}[1]{\textnormal{\scshape #1}}

\makeatletter
\def\blfootnote{\gdef\@thefnmark{}\@footnotetext}
\makeatother

\newcommand{\zoomfig}[5]{%
\begin{tikzpicture}[spy using outlines={
    rectangle,
    magnification=3,
    size=#5,
    connect spies,
    every spy on node/.append style={red, dashed, line width=0.8pt},
    every spy in node/.append style={red, dashed, line width=0.8pt}
}]
    \node[inner sep=0] (img) {\includegraphics[width=#4]{#1}};
    
    \spy on ($(img.south west)!#2!(img.south east)!#3!(img.north west)$)
        in node[anchor=north west] at ($(img.north east)+(-0.05cm,-0.05cm)$);
\end{tikzpicture}%
}

\begin{document}

%%
%% Title
%%
\title{SHIFT: Stochastic Hidden-Trajectory Deflection for Removing Diffusion-based Watermarks}

%%
%% Authors - TODO: Replace with your actual author information
%%
\author{
  \begin{tabular}{ccc}
Rui Bao$^{1}$\footnote{These authors contributed equally to this work.} ~&~ Zheng Gao$^{1*}$ ~&~ Xiaoyu Li$^1$ \\
 Xiaoyan Feng$^2$ ~&~ Yang Song$^1$ ~&~ Jiaojiao Jiang$^1$ \\[2ex]  
  \end{tabular}\\[1ex]  
  \begin{tabular}{ccc}
    \multicolumn{3}{c}{$^1$University of New South Wales} \\[1ex]  
    \multicolumn{3}{c}{$^2$Griffith University}
    % \multicolumn{3}{c}{\small $^*$These authors contributed equally to this work.}
  \end{tabular}
}

\maketitle

\begin{abstract}
\input{1_abstract}

\end{abstract}

\input{2_intro}
\input{3_relatedwork}

\input{4_method}
\input{5_experiments}

\input{6_conclusion}

\bibliographystyle{splncs04}
\bibliography{ref}

\clearpage

\appendix

\begin{center}
    \textbf{\huge Appendix}
\end{center}
\input{10_app}

\end{document}

%% file: 1_abstract.tex
Diffusion-based watermarking methods embed verifiable marks by manipulating the initial noise or the reverse diffusion trajectory. However, these methods share a critical assumption: verification can succeed only if the diffusion trajectory can be faithfully reconstructed. This reliance on trajectory recovery constitutes a fundamental and exploitable vulnerability. We propose \underline{\textbf{S}}tochastic \underline{\textbf{Hi}}dden-Trajectory De\underline{\textbf{f}}lec\underline{\textbf{t}}ion (\textbf{SHIFT})\footnote{Our code is available at \href{https://github.com/ZhengGao-30/SHIFT-Watermark-Attack}{https://github.com/ZhengGao-30/SHIFT-Watermark-Attack}.}, a training-free attack that exploits this common weakness across diverse watermarking paradigms. SHIFT leverages stochastic diffusion resampling to deflect the generative trajectory in latent space, making the reconstructed image statistically decoupled from the original watermark-embedded trajectory while preserving strong visual quality and semantic consistency. Extensive experiments on nine representative watermarking methods spanning noise-space, frequency-domain, and optimization-based paradigms show that SHIFT achieves 95\%--100\% attack success rates with nearly no loss in semantic quality, without requiring any watermark-specific knowledge or model retraining.

\blfootnote{Corresponding to: \texttt{zheng.gao1@unsw.edu.au}, \texttt{jiaojiao.jiang@unsw.edu.au}}

%% file: 2_intro.tex
\section{Introduction}
Diffusion-based generative models can now synthesize high-fidelity images that are virtually indistinguishable from real photographs. However, their potential for misuse in misinformation dissemination, copyright infringement, and deepfake generation has made reliable provenance tracking of AI-generated content a pressing concern. Digital watermarking, as a proactive content tracing mechanism, is widely regarded as one of the most promising technical solutions.

In recent years, a new paradigm of \emph{diffusion watermarking} has attracted considerable attention. Unlike traditional post-processing watermarks, these methods embed watermarks directly into the generative process of diffusion models: by modifying the frequency-domain structure of the initial noise vector (e.g., Tree-Ring~\cite{Tree-Ring}, RingID~\cite{ci2024ringid}), constraining the sampling distribution of latent variables (e.g., Gaussian Shading~\cite{yang2024gaussian}, PRC~\cite{gunn2024undetectable}), or binding semantic hashes (e.g., SEAL~\cite{arabi2025seal}), such that the watermark is deeply coupled with the generated content at the latent-space level. During verification, the defender maps the suspect image back to its initial latent variable via deterministic DDIM inversion and checks for the presence of the watermark signal. Owing to this generation-level embedding mechanism, diffusion watermarks exhibit excellent robustness against a wide range of conventional attacks, including JPEG compression, Gaussian blur, cropping, and rotation.

Meanwhile, attack methods targeting watermarks have also been rapidly evolving. Zhao et al.~\cite{zhao2024invisible} employs a deterministic PNDM sampler to reconstruct images from noisy latent representations, and theoretically proves that it can remove any pixel-level watermark with bounded $\ell_2$ perturbation. The latent-noise removing attack~\cite{jain2025forging} exploits the many-to-one mapping between images and initial noises, forging and removing watermarks by iteratively optimizing adversarial perturbations in the VAE latent space. Black-box attack~\cite{muller2025black} and D$^2$RA~\cite{meshram2025d2ra} further develop attack strategies from the perspectives of latent-space optimization and dual-domain frequency reconstruction, respectively. However, these methods face a common systemic predicament. The Regen attack demonstrates that regeneration based on deterministic samplers is nearly ineffective against semantic watermarks such as Tree-Ring, with post-attack detection rates remaining close to 100\%. D$^2$RA similarly finds that pixel-domain regeneration alone cannot remove Tree-Ring watermarks. While the latent-noise removing attack and black-box method can attack certain semantic watermarks, they require costly per-image iterative optimization that is difficult to scale. The fundamental reason is that deterministic regeneration inevitably preserves the structural frequency information in the latent space that carries the watermark, while existing methods that bypass this limitation must resort to expensive per-sample optimization. This systemic predicament has led diffusion watermarks to be regarded as the most reliable solution for AI content provenance.

However, this paper reveals a fundamental security vulnerability that has been entirely overlooked in the current literature. All of the aforementioned diffusion watermarks, regardless of their specific embedding mechanisms, implicitly rely on the same critical assumption during verification: \textbf{trajectory consistency}, i.e., the mathematical symmetry between the forward generation path and the backward DDIM inversion path can be exactly preserved. We discover that this shared dependence on trajectory consistency constitutes a unified and exploitable attack surface. Our key insight is that \textbf{stochastic reverse sampling inherently breaks this trajectory consistency}. Specifically, the stochastic sampling process injects independent Brownian motion noise at each timestep, irreversibly deflecting the generative trajectory from the original deterministic path. Meanwhile, the pretrained score function continues to pull the sample toward the natural image manifold, thereby preserving high semantic fidelity in the output image, yet the latent-space path it traverses has been thoroughly decoupled from the original watermark trajectory. When the defender performs DDIM inversion on the attacked image, the inversion traces back to a latent variable that is entirely unrelated to the embedded watermark, causing verification to fail. Based on this finding, we propose \textbf{SHIFT}    (\underline{\textbf{S}}tochastic \underline{\textbf{Hi}}dden-Trajectory De\underline{\textbf{f}}lec\underline{\textbf{t}}ion), a training-free attack that exploits stochastic reverse resampling to deflect the generative trajectory and thereby uniformly defeat diffusion watermarks.

SHIFT consists of two stages: (i)~partial forward diffusion, which injects the target watermarked image into an intermediate noisy state to progressively attenuate trajectory-specific information; and (ii)~stochastic reverse resampling, which regenerates the image from this noisy state along an entirely new random trajectory. SHIFT requires no model retraining or fine-tuning, no knowledge of the specific watermarking scheme, and no adversarial optimization, but only a publicly available pretrained diffusion model. On the theoretical side, we formalize trajectory-decoupling guarantees based on Wasserstein distance analysis, rigorously proving that the noise recovered via DDIM inversion from the attacked image is statistically independent of the original watermark noise.

In summary, our main contributions are as follows:
\begin{itemize}
    \item We identify \emph{trajectory consistency} as a fundamental 
    vulnerability shared by all existing diffusion watermarking 
    schemes, providing a unified perspective for analyzing their 
    security.
    
    \item We propose SHIFT, a training-free attack that deflects the 
    generative trajectory via partial forward diffusion and 
    stochastic reverse resampling, requiring no watermark-specific 
    knowledge or adversarial optimization.
    
    \item We provide formal trajectory-decoupling guarantees via 
    Wasserstein distance analysis, theoretically proving that 
    stochastic resampling breaks watermark verification while 
    deterministic resampling cannot.
    
    \item Experiments on nine watermarking methods across three 
    paradigms show that SHIFT achieves 95\%--100\% attack success 
    rates with the best image quality among all compared attacks.
\end{itemize}

\paragraph{Roadmap.} In Section~\ref{sec:related}, we discuss related work on diffusion-based watermarking, watermark removal attacks, and diffusion sampling. Section~\ref{sec:method} introduces SHIFT and describes its two-stage attack pipeline. Section~\ref{sec:theory} presents the theoretical analysis underlying trajectory decoupling. Section~\ref{sec:exp} evaluates SHIFT on a diverse set of watermarking methods and examines both attack effectiveness and perceptual quality. Section~\ref{sec:conclusion} concludes with limitations and directions for future work. Proofs and additional experimental results are provided in the appendix.

%% file: 3_relatedwork.tex
\section{Related Work}\label{sec:related}

\subsection{Diffusion Models and Sampling Mechanisms}
Diffusion models~\cite{ho2020denoising,song2019generative,song2021sde}
learn to reverse a forward noising process to generate samples from data
distributions. Latent Diffusion Models
(LDMs)~\cite{rombach2022high,kingma2014vae,esser2021vqgan} further
perform this denoising within a compressed latent space encoded by a
pretrained VAE, significantly reducing computational cost. A critical
design axis in diffusion sampling is the choice between deterministic
and stochastic reverse processes. Deterministic samplers such as
DDIM~\cite{song2021ddim} and its higher-order
extensions~\cite{lu2022dpmsolver,zhao2023unipc,liu2022pndm} establish
a bijective mapping from a given initial noise to a unique generated
image, a property that underpins DDIM
inversion~\cite{mokady2023nulltext,wallace2023edict} and, by
extension, most diffusion-based watermark verification schemes. However,
this deterministic invertibility is inherently
fragile~\cite{lin2024schedule,blasingame2025reversible}: prediction
errors accumulate along the trajectory, and the resulting reconstruction
inconsistency has been identified as a systematic vulnerability. In
contrast, stochastic samplers~\cite{ho2020denoising,song2021sde,
xue2023sasolver,gonzalez2023seeds} inject fresh Brownian noise at every
reverse step, causing distinct sampling runs from the same starting
point to diverge onto entirely different generative trajectories. Recent
analysis~\cite{nie2024sdebeatsode} further demonstrates that this
stochasticity is advantageous for image-to-image tasks. It is precisely
this trajectory-diversifying property that our attack exploits.

\subsection{Watermarking for Diffusion Models}
As diffusion models become the dominant paradigm for image generation,
a growing body of work has sought to embed verifiable watermarks into
the generation pipeline itself. Existing schemes can be broadly
grouped by where the watermark is injected. \emph{Noise-space methods}
modify the initial latent before sampling:
Tree-Ring~\cite{Tree-Ring} writes a Fourier ring pattern into the
initial noise, RingID~\cite{ci2024ringid} extends this to
multi-channel heterogeneous patterns for multi-key identification,
PRC~\cite{gunn2024undetectable} draws initial latents from pseudorandom error-correcting
codes for provable undetectability, and WIND~\cite{arabi2024hidden}
organises large key pools through a two-stage grouping framework.
\emph{Latent- and frequency-domain methods} instead act on
intermediate representations: Gaussian
Shading~\cite{yang2024gaussian} introduces key-controlled
colour offsets in the spectral domain during sampling,
GaussMarker~\cite{li2025gaussmarker} jointly encodes bits in
high-frequency components and the noise trajectory, while
SFW~\cite{lee2025semantic} and SEAL~\cite{arabi2025seal} bind the mark to
semantic content. More recently, SLICE~\cite{gao2026slice} enables localized tamper detection via compartmentalized semantic watermarking. \emph{Optimisation-based methods} such as
ROBIN~\cite{huang2024robin} take a different route, training hidden
prompts via adversarial objectives so that generated images
automatically carry a latent-space signature.
Although these schemes differ in design, they share a common
assumption: verification succeeds only when the diffusion trajectory
can be faithfully traced back, either through exact noise inversion
or through the stability of intermediate latent structures. It is
precisely this trajectory dependence that our attack targets.

\subsection{Attacks on Image Watermarks}
Existing attacks on image watermarks include regeneration-based and optimization-based approaches. \emph{Regeneration-based attacks} reconstruct the suspect image through a pretrained diffusion model to disrupt the embedded signal. Zhao et al.~\cite{zhao2024invisible} formalize this paradigm with a deterministic PNDM sampler and prove removal guarantees for pixel-level watermarks. CtrlRegen~\cite{liu2024image} augments the pipeline with trained semantic and spatial control modules to improve fidelity. Saberi et al.~\cite{saberi2023robustness} adapt DiffPure to watermark removal using a reverse-time SDE in pixel space, while DDWRM~\cite{mareen2024diffusion} applies pixel-space DDPM denoising to traditional watermarks. \emph{Optimization-based attacks} take a different route: the latent-noise removing attack~\cite{jain2025forging} and black-box attack~\cite{muller2025black} iteratively optimize adversarial perturbations in the VAE latent space, achieving targeted evasion at the cost of expensive per-image gradient computation. CSI~\cite{gao2026breakingsemanticawarewatermarksllmguided} exploit LLM-guided semantic manipulation in discrete prompt space to evade content-aware watermarks such as SEAL. Despite their diversity, all these methods treat watermark robustness as an empirical obstacle and target individual schemes without analyzing the common mechanism that underlies verification. SHIFT instead identifies trajectory consistency as the shared vulnerability across all diffusion watermarking paradigms and exploits it through a principled, training-free attack with formal trajectory-decoupling guarantees.

%% file: 4_method.tex
\section{Method}\label{sec:method}

\begin{figure*}[t]
    \centering
    \includegraphics[width=0.8\textwidth]{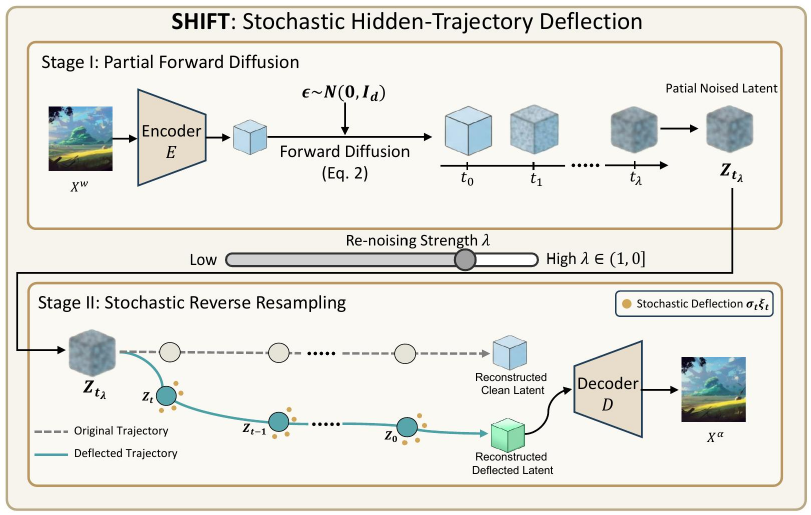}
    \caption{The overall framework of SHIFT}
    \label{fig:framework}
\end{figure*}

We propose \underline{\textbf{S}}tochastic \underline{\textbf{Hi}}dden-Trajectory De\underline{\textbf{f}}lec\underline{\textbf{t}}ion (\textbf{SHIFT}), a training-free attack against diffusion-based watermarking. Figure~\ref{fig:framework} provides an overview of the proposed framework. The key insight is that these schemes, whether they manipulate initial noise, intermediate latents, or reverse-process trajectories, all rely on the same assumption at verification time: the diffusion trajectory can be faithfully reconstructed from the generated image.
SHIFT breaks this assumption by \emph{deflecting} the generative trajectory via stochastic resampling, producing an image that is semantically faithful to the original yet statistically decoupled from the watermark-embedded path.

\subsection{Threat Model}

We consider a realistic \emph{black-box} attack setting formalized as follows.

\begin{definition}[Watermarked Generation Pipeline]
Let $\mathcal{W}$ denote a diffusion-based watermarking scheme that, given a text prompt $c$ and a secret key $k$, produces a watermarked image $\mathbf{x}^w = \mathcal{W}(c, k)$ by manipulating the diffusion process (e.g., initial noise, intermediate latents, or reverse trajectory).
An associated verifier $\mathcal{V}_k$ extracts a message estimate $\hat{m} = \mathcal{V}_k(\mathbf{x}^q)$ from any query image $\mathbf{x}^q$ and declares \textsc{Watermarked} if $\operatorname{BA}(\hat{m}, m) \geq \tau$, where $\operatorname{BA}$ denotes the bit accuracy and $\tau$ is a verification threshold.
\end{definition}

\begin{assumption}[Attacker's Capabilities and Knowledge]
\label{asm:attacker} 
We assume that the attacker has access to:
\begin{enumerate}[label=(\roman*),nosep]
    \item the watermarked image $\mathbf{x}^w$;
    \item a \emph{publicly available} pretrained latent diffusion model $(\mathcal{E}, \mathcal{D}, \boldsymbol{\epsilon}_\theta)$ with VAE encoder $\mathcal{E}$, decoder $\mathcal{D}$, and noise prediction network $\boldsymbol{\epsilon}_\theta$. It may or may not be the same model used by $\mathcal{W}$.
\end{enumerate}
The attacker does \textbf{not} know:
\begin{enumerate}[label=(\roman*),nosep,resume]
    \item the watermark key $k$, message $m$, or embedding algorithm;
    \item the prompt $c$ used during generation;
    \item the internal architecture or parameters of the verifier $\mathcal{V}_k$.
\end{enumerate}
In particular, no model retraining, adversarial optimization, or watermark-specific adaptation is permitted.
\end{assumption}

\subsection{Attack Objective}

Given a watermarked image $\mathbf{x}^w$, the attacker seeks to produce an attacked image $\mathbf{x}^a = \mathcal{A}(\mathbf{x}^w)$ that simultaneously minimizes two competing quantities:
\begin{enumerate}[nosep]
    \item \textbf{Semantic  Fidelity:} the perceptual distance $d_{\mathrm{sem}}(\mathbf{x}^a, \mathbf{x}^w)$ (e.g., LPIPS~\cite{zhang2018unreasonable} or $1 - \mathrm{CLIP}_{\mathrm{sim}}$), ensuring that the attacked image remains visually faithful to the original;
    \item \textbf{Bit Accuracy:} $\operatorname{BA}\bigl(\mathcal{V}_k(\mathbf{x}^a),\, m\bigr)$, driving the extracted watermark away from the embedded message $m$ so that verification fails.
\end{enumerate}

The central challenge is that pixel-level perturbations (e.g., JPEG compression, Gaussian blur) trade fidelity for evasion in a crude, watermark-agnostic manner.
We instead operate in the \emph{latent trajectory space} of diffusion models, where watermark information is structurally encoded, enabling targeted disruption of the verification signal while preserving semantic content.

\subsection{Stage I: Partial Forward Diffusion}

Diffusion-based watermarks encode information into the generative trajectory, either through the initial noise $\mathbf{z}_T$, intermediate latents $\{\mathbf{z}_t\}_{t}$, or the mapping between them.
To disrupt this encoding, we must erase trajectory-specific information from the latent representation of $\mathbf{x}^w$.
A naive approach would add pixel-space noise, but this destroys high-level semantic structure without targeting the latent trajectory. Instead, we we perform \emph{controlled forward diffusion in latent space}, which progressively replaces trajectory-specific fine-grained details with isotropic Gaussian noise while preserving the coarse semantic scaffold encoded in the signal component. 

\paragraph{Procedure.} Given the watermarked image $\mathbf{x}^w$, we first obtain its latent representation via the pretrained VAE encoder:
\begin{equation}
    \mathbf{z}_0 = \mathcal{E}(\mathbf{x}^w).
    \label{eq:encode}
\end{equation}
We then select a re-noising strength $\lambda \in (0,1]$ and define the target timestep $
t_\lambda = \lceil \lambda T \rceil,
$
where $T$ is the total number of diffusion steps.
The partially noised latent is constructed via the closed-form DDPM forward process~\cite{ho2020denoising}:
\begin{equation}
    \mathbf{z}_{t_\lambda} = \sqrt{\bar{\alpha}_{t_\lambda}}\, \mathbf{z}_0 + \sqrt{1 - \bar{\alpha}_{t_\lambda}}\, \boldsymbol{\epsilon},
    \label{eq:forward}
\end{equation}
where $\boldsymbol{\epsilon} \sim \mathcal{N}(\mathbf{0}, \mathbf{I}_d)$ is the standard Gaussian noise, and $\bar{\alpha}_{t_\lambda} = \prod_{i=1}^{t_\lambda}(1 - \beta_i)$ is the cumulative signal retention coefficient under the noise schedule $\{\beta_i\}_{i=1}^T$.

\begin{remark}[Role of $\lambda$]
\label{rem:lambda}
The parameter $\lambda$ governs the \emph{fidelity--evasion trade-off}.
As $\lambda$ increases, the target timestep $t_\lambda$ moves deeper into the same prescribed $T$-step diffusion chain, so the partially noised latent retains less trajectory-specific information from the source latent $\mathbf{z}_0$.
With the discrete convention $t_\lambda = \lceil \lambda T \rceil$, we have $t_\lambda \to T$ as $\lambda \to 1$, so $\lambda \to 1$ corresponds to approaching the terminal depth of the original diffusion schedule.
Larger $\lambda$ therefore yields stronger trajectory removal but increases the risk of semantic degradation.
The optimal $\lambda$ is the smallest value that achieves watermark evasion.
\end{remark}

\subsection{Stage II: Stochastic Reverse Resampling}

Given the partially noised latent $\mathbf{z}_{t_\lambda}$, we must reconstruct a clean image.
A deterministic sampler (e.g., DDIM~\cite{song2021ddim}) would attempt to recover the \emph{unique} trajectory associated with $\mathbf{z}_{t_\lambda}$, potentially preserving residual watermark structure.
Instead, we employ a \emph{stochastic} sampler, which injects fresh Brownian noise at each reverse step.
This ensures that the reverse trajectory is not uniquely determined by the starting point, causing the regenerated sample to diverge onto an entirely new generative path.

\paragraph{Procedure.} We perform reverse diffusion from $\mathbf{z}_{t_\lambda}$ to $\mathbf{z}'_0$ using the Euler ancestral sampler. At each reverse step from timestep $t$ to $t{-}1$, the denoiser first predicts the clean latent:
\begin{equation}
    \hat{\mathbf{z}}_0 = \frac{\mathbf{z}_t - \sqrt{1 - \bar{\alpha}_t}\, \boldsymbol{\epsilon}_\theta(\mathbf{z}_t, t)}{\sqrt{\bar{\alpha}_t}},
    \label{eq:predict_clean}
\end{equation}
and the update rule is
\begin{equation}
    \mathbf{z}_{t-1} = \sqrt{\bar{\alpha}_{t-1}}\, \hat{\mathbf{z}}_0 
    + \underbrace{\sqrt{1 - \bar{\alpha}_{t-1} - \sigma_t^2}\, \boldsymbol{\epsilon}_\theta(\mathbf{z}_t, t)}_{\text{deterministic direction}}
    + \underbrace{\sigma_t\, \boldsymbol{\xi}_t}_{\text{stochastic deflection}}
    \label{eq:reverse}
\end{equation}
where $\boldsymbol{\xi}_t \sim \mathcal{N}(\mathbf{0}, \mathbf{I}_d)$  is the standard Gaussian noise, and $\sigma_t > 0$ is the ancestral noise scale.
When $\sigma_t = 0$, this reduces to deterministic DDIM; when $\sigma_t > 0$, each step injects fresh randomness that progressively deflects the trajectory.

The final latent $\mathbf{z}_0$ is decoded to pixel space:
\begin{equation}
    \mathbf{x}^a = \mathcal{D}(\mathbf{z}_0).
    \label{eq:decode}
\end{equation}

\subsection{Algorithm}
The complete algorithm of SHIFT is presented in Algorithm~\ref{alg:shift}. 
The attack requires only a single forward pass through the VAE encoder, $t_\lambda$ evaluations of $\boldsymbol{\epsilon}_\theta$, and a single VAE decoder. Hence it has the same cost as a standard img2img diffusion pipeline with $t_\lambda$ steps.
No gradient computation, iterative optimization, or watermark-specific adaptation is needed, making SHIFT both computationally efficient and universally applicable.

\begin{algorithm}[!htp]
\caption{SHIFT: Stochastic Hidden-Trajectory Deflection}
\label{alg:shift}
\begin{algorithmic}[1]
\Require Watermarked image $\mathbf{x}^w$, pretrained LDM $(\mathcal{E}, \mathcal{D}, \boldsymbol{\epsilon}_\theta)$, re-noising strength $\lambda \in (0,1)$, noise schedule $\{\beta_i\}_{i=1}^T$, ancestral noise scales $\{\sigma_t\}$
\Ensure Attacked image $\mathbf{x}^a$
\Statex
\Statex \hspace{-1.5em}\textbf{Stage I: Partial Forward Diffusion}
\State $\mathbf{z}_0 \gets \mathcal{E}(\mathbf{x}^w)$ \Comment{Encode to latent space}
\State $t_\lambda \gets \lfloor \lambda \cdot T \rfloor$ \Comment{Target re-noising timestep}
\State $\boldsymbol{\epsilon} \sim \mathcal{N}(\mathbf{0}, \mathbf{I}_d)$ \Comment{Sample fresh noise}
\State $\mathbf{z}_{t_\lambda} \gets \sqrt{\bar{\alpha}_{t_\lambda}}\, \mathbf{z}_0 + \sqrt{1 - \bar{\alpha}_{t_\lambda}}\, \boldsymbol{\epsilon}$ \Comment{Controlled re-noising}
\Statex
\Statex \hspace{-1.5em}\textbf{Stage II: Stochastic Reverse Resampling}
\For{$t = t_\lambda, t_\lambda - 1, \ldots, 1$}
    \State $\hat{\mathbf{z}}_0 \gets \bigl(\mathbf{z}_t - \sqrt{1 - \bar{\alpha}_t}\, \boldsymbol{\epsilon}_\theta(\mathbf{z}_t, t)\bigr) \big/ \sqrt{\bar{\alpha}_t}$ \Comment{Predict clean latent}
    \State $\boldsymbol{\xi}_t \sim \mathcal{N}(\mathbf{0}, \mathbf{I}_d)$ \Comment{Sample fresh noise}
    \State $\mathbf{z}_{t-1} \gets \sqrt{\bar{\alpha}_{t-1}}\, \hat{\mathbf{z}}_0 + \sqrt{1 - \bar{\alpha}_{t-1} - \sigma_t^2}\, \boldsymbol{\epsilon}_\theta(\mathbf{z}_t, t) + \sigma_t\, \boldsymbol{\xi}_t$ 
    \Statex \Comment{Ancestral update}
\EndFor
\Statex
\Statex \hspace{-1.5em}\textbf{Decode}
\State $\mathbf{x}^a \gets \mathcal{D}(\mathbf{z}_0)$ \Comment{Decode to pixel space}
\State \Return $\mathbf{x}^a$
\end{algorithmic}
\end{algorithm}

\section{Theoretical Justification on Trajectory Decoupling}
\label{sec:theory}

Diffusion-based watermark verifiers implicitly rely on a \emph{trajectory-consistency} assumption: the queried image remains statistically coupled to the watermark-carrying generative path, so that inversion can recover noise or latent features aligned with the embedded signal.
SHIFT breaks this assumption.
Stage I attenuates the contribution of the source latent via partial forward diffusion, while Stage II reconstructs the image through stochastic ancestral resampling, thereby replacing the original watermark-consistent reverse path with a newly sampled one.
Thus, the attack does not merely perturb pixels; it disrupts the trajectory-level dependency on which diffusion watermark verification depends.

To formalize this effect, let $\hat{\boldsymbol{\epsilon}}_{t_\lambda}^a$ denote the DDIM-inverted noise recovered from the attacked image generated by SHIFT at depth $t_\lambda = \lceil \lambda T \rceil$, and let $\tilde{\boldsymbol{\epsilon}}_{t_\lambda}$ denote the baseline recovered noise obtained by running the same stochastic reverse process from pure attack noise (i.e., with the source-latent component removed), so that $\tilde{\boldsymbol{\epsilon}}_{t_\lambda}$ is independent of the watermark-carrying noise $\boldsymbol{\epsilon}^w$.
The full formal theorem and proof are deferred to Appendix.
Here, $W_2$ denotes the $2$-Wasserstein distance, $\mathcal{L}(\cdot)$ denotes the probability law of a random variable, and $\otimes$ denotes the product of marginal laws.

\begin{theorem}[Informal]
\label{prop:shift_informal}
Under mild regularity assumptions on the public denoiser, the stochastic reverse sampler, and the DDIM inversion map, the following hold for every $\lambda\in(0,1]$:
\begin{enumerate}[label=(\roman*),nosep,leftmargin=*]
    \item It holds that
    $$\mathbb{E}\left[\|\hat{\boldsymbol{\epsilon}}_{t_\lambda}^a-\tilde{\boldsymbol{\epsilon}}_{t_\lambda}\|_2^2\right]
    \le \Delta_{t_\lambda}, ~~~\text{where}~~~\Delta_{t_\lambda} := L_Q^2 C_{t_\lambda}^2\,\bar{\alpha}_{t_\lambda}\,\mathbb{E}[\|\mathbf{z}_0\|_2^2].$$
    Here $C_{t_\lambda}$ is the cumulative Lipschitz constant of the $t_\lambda$-step reverse map, and $L_Q$ is the Lipschitz constant of the verification pipeline ($\mathcal{I}_{\mathrm{DDIM}} \circ \mathcal{E} \circ \mathcal{D}$);

    \item It holds that $$
    W_2\left(
    \mathcal{L}(\hat{\boldsymbol{\epsilon}}_{t_\lambda}^a,\boldsymbol{\epsilon}^w),\,
    \mathcal{L}(\hat{\boldsymbol{\epsilon}}_{t_\lambda}^a)\otimes \mathcal{L}(\boldsymbol{\epsilon}^w)
    \right)
    \le 2\sqrt{\Delta_{t_\lambda}};
    $$

    \item
    If, in addition, $\tilde{\boldsymbol{\epsilon}}_T \sim \mathcal{N}(\mathbf{0},\mathbf{I}_d)$ and $\boldsymbol{\epsilon}^w \sim \mathcal{N}(\mathbf{0},\mathbf{I}_d)$, then at the terminal step $t_\lambda = T$,
    \[
    \left|
    \mathbb{E}\left[\|\hat{\boldsymbol{\epsilon}}_T^a-\boldsymbol{\epsilon}^w\|_2^2\right]
    - 2d
    \right|
    \le
    \Delta_T + 2\sqrt{2d\,\Delta_T}.
    \]
\end{enumerate}
\end{theorem}

\noindent\textbf{Takeaway.}
Theorem~\ref{prop:shift_informal} formalizes the trajectory-level effect of SHIFT at any attack depth, not only at the terminal step.
The bound $\Delta_{t_\lambda} = L_Q^2 C_{t_\lambda}^2\,\bar{\alpha}_{t_\lambda}\,\mathbb{E}[\|\mathbf{z}_0\|_2^2]$ captures two competing effects: as $t_\lambda$ increases, $\bar{\alpha}_{t_\lambda}$ decreases (stronger erasure of the source signal), while $C_{t_\lambda}$ generally grows (error accumulation over more reverse steps).
For each watermarking method, the attack succeeds once $\Delta_{t_\lambda}$ falls below the effective verification tolerance, and this critical $\lambda$ varies across methods.
This is consistent with our experimental findings (Table~\ref{tab:main_results}): methods with fragile trajectory dependence such as SEAL and SFW are defeated at $\lambda \le 0.15$, while methods that distribute watermark information more deeply, such as Tree-Ring and Gaussian Shading, require $\lambda \ge 0.50$.

Part (ii) sharpens the decoupling statement: the attacked recovered noise $\hat{\boldsymbol{\epsilon}}_{t_\lambda}^a$ is nearly independent of the original watermark noise $\boldsymbol{\epsilon}^w$, as measured by the Wasserstein distance between their joint law and the corresponding product of marginals.
Part (iii) specializes to the terminal step under an additional Gaussianity assumption: the recovered noise becomes nearly as far from the watermark noise as an independent random draw, yielding the random baseline $2d$.
In short, SHIFT defeats verification by destroying trajectory consistency, rather than by crudely degrading the image in pixel space.

\begin{remark}[Why Stochasticity Is Essential]
\label{rem:stochasticity}
The effectiveness of SHIFT does not come from re-noising alone, but from re-noising followed by stochastic resampling.
If Stage II were replaced by a deterministic reverse process such as DDIM, then conditioned on $\mathbf{z}_{t_\lambda}$, the reconstructed latent $\mathbf{z}_0'$ would be uniquely determined:
\[
\mathcal{L}(\mathbf{z}_0' \mid \mathbf{z}_{t_\lambda})
=
\delta_{G_\lambda(\mathbf{z}_{t_\lambda})}
\]
for some deterministic map $G_\lambda$, where $\delta$ is Dirac delta measure.
However, the partially noised latent still contains a nonzero source-dependent component,
\[
\mathbf{z}_{t_\lambda}
=
\sqrt{\bar{\alpha}_{t_\lambda}}\,\mathbf{z}_0
+
\sqrt{1-\bar{\alpha}_{t_\lambda}}\,\boldsymbol{\epsilon},
\]
so a deterministic reverse map can preserve residual dependence on $\mathbf{z}_0$, and hence on the watermark-carrying trajectory.

By contrast, ancestral sampling with $\sigma_t>0$ injects fresh Gaussian noise at every reverse step:
\[
\mathbf{z}_{t-1}
=
\sqrt{\bar{\alpha}_{t-1}}\,\hat{\mathbf{z}}_0
+
\sqrt{1-\bar{\alpha}_{t-1}-\sigma_t^2}\,
\boldsymbol{\epsilon}_\theta(\mathbf{z}_t,t)
+
\sigma_t\boldsymbol{\xi}_t,
\qquad
\boldsymbol{\xi}_t\sim\mathcal{N}(\mathbf{0},\mathbf{I}_d).
\]
As a result, the conditional law $\mathcal{L}(\mathbf{z}_0' \mid \mathbf{z}_{t_\lambda})$ is no longer degenerate: the final reconstruction depends not only on the starting latent $\mathbf{z}_{t_\lambda}$, but also on the $t_\lambda$ fresh $d$-dimensional perturbations $\{\boldsymbol{\xi}_t\}_{t=1}^{t_\lambda}$.
This breaks the one-to-one correspondence between the partially noised latent and the reconstructed reverse trajectory.
Operationally, Stage I suppresses the original trajectory signal, while Stage II replaces it with newly sampled trajectory randomness.
This is precisely why SHIFT succeeds at lower $\lambda$ than attacks based solely on re-noising or deterministic resampling.
\end{remark}

%% file: 5_experiments.tex
\section{Experiments}\label{sec:exp}
We adopt Stable Diffusion v2.1 as the base diffusion model for SHIFT. The evaluation prompts are drawn from the publicly available Stable-Diffusion-Prompts dataset~\cite{Santana2024StableDiffusionPrompts}. Unless otherwise specified, all experiments are conducted on an H100 GPU.

\textbf{Compared methods.} We evaluate SHIFT against nine representative diffusion watermarking methods, including Tree-Ring~\cite{Tree-Ring}, RingID~\cite{ci2024ringid}, PRC~\cite{gunn2024undetectable}, WIND~\cite{arabi2024hidden}, Gaussian Shading~\cite{yang2024gaussian}, GaussMarker~\cite{li2025gaussmarker}, SFW~\cite{lee2025semantic}, SEAL~\cite{arabi2025seal}, and ROBIN~\cite{huang2024robin}. In addition, we compare SHIFT with two existing attack methods, namely black-box attack~\cite{muller2025black} and removing attack~\cite{jain2025forging}, to comprehensively evaluate its attack effectiveness and its impact on image quality.

\textbf{Evaluation metrics.} We use the attack success rate (ASR, $\uparrow$) to measure the ability of an attack to disrupt watermark verification, and adopt the CLIP score ($\uparrow$)~\cite{hessel2021clipscore} and FID ($\downarrow$)~\cite{heusel2017gans} to evaluate the semantic consistency and distribution quality of the attacked images. In further analysis experiments, we also examine the changes in the distance between the recovered noise and the original watermark-carrying noise before and after the attack.

\begin{table}[t]
\centering
\definecolor{vulnhigh}{RGB}{227,242,253}
\definecolor{vulnmid}{RGB}{255,243,224}
\definecolor{vulnlow}{RGB}{232,245,233}

\caption{
\textbf{Overall attack results.}
All values are ASR (\%).}
\label{tab:main_results}
\vspace{1mm}
\setlength{\tabcolsep}{4pt}
\renewcommand{\arraystretch}{1.12}
\small
\begin{tabular}{lcccc}
\toprule
\textbf{Method} & \textbf{Clean} & \textbf{Black-box\cite{muller2025black}} & \textbf{Remove\cite{jain2025forging}} & \textbf{SHIFT} \\
\midrule

\multicolumn{5}{l}{\textit{\small Highly vulnerable} ($\lambda \leq 0.15$)} \\[1pt]
\rowcolor{vulnhigh}
SEAL~\cite{arabi2025seal}                & 0 & 100 & 100 & 98 \\
\rowcolor{vulnhigh}
SFW~\cite{lee2025semantic}               & 0 & 100 & 100 & 100 \\
\rowcolor{vulnhigh}
WIND~\cite{arabi2024hidden}              & 0 & 100 & 100 & 98 \\
\rowcolor{vulnhigh}
PRC~\cite{gunn2024undetectable}          & 0 & 100 & 100 & 97 \\[2pt]

\multicolumn{5}{l}{\textit{\small Moderately robust} ($\lambda = 0.30$)} \\[1pt]
\rowcolor{vulnmid}
RingID~\cite{ci2024ringid}               & 0 & 99 & 95 & 95 \\
\rowcolor{vulnmid}
ROBIN~\cite{huang2024robin}              & 0 & 100 & 79 & 98 \\[2pt]

\multicolumn{5}{l}{\textit{\small Most robust} ($\lambda \geq 0.50$)} \\[1pt]
\rowcolor{vulnlow}
Tree\mbox{-}Ring~\cite{Tree-Ring}        & 0 & 72 & 99 & 98 \\
\rowcolor{vulnlow}
GaussMarker~\cite{li2025gaussmarker}     & 0 & 100 & 55 & 99 \\
\rowcolor{vulnlow}
Gaussian Shading~\cite{yang2024gaussian} & 0 & 100 & 4 & 97 \\
\midrule
\textbf{Average}                         & --- & 96.8 & 81.3 & 97.8 \\
\bottomrule
\end{tabular}

\vspace{-2mm}
\end{table}

\subsection{Attack Performance Against Diffusion Watermarks}

We first evaluate the overall attack capability of SHIFT against different diffusion watermarking methods. Table~\ref{tab:main_results} reports the attack success rate (ASR) on nine representative diffusion watermarking methods, where, for each method, we use the smallest re-noising strength $\lambda$ that successfully bypasses verification. We compare SHIFT with the clean case and two existing attack methods, namely black-box attack and removing attack. The results show that SHIFT consistently achieves strong attack performance across all nine methods, with ASR ranging from 95\% to 100\% and an average ASR of 97.8\%, outperforming both black-box attack (96.8\%) and removing attack (81.3\%). Meanwhile, all methods have an ASR of 0 under the clean setting, indicating that watermark verification is reliable before attack. These results demonstrate that SHIFT not only effectively breaks watermark verification, but also surpasses existing attacks in overall effectiveness.

A closer examination across individual methods further shows that SHIFT is more stable and generalizable than prior attacks. Its advantage is particularly evident on more challenging methods: on Tree-Ring, SHIFT achieves 98\%, substantially higher than the 72\% of black-box attack; on GaussMarker, SHIFT reaches 99\%, clearly outperforming the 55\% of removing attack; and on Gaussian Shading, SHIFT still attains 97\%, whereas removing attack achieves only 4\%. The minimum attack strength required for successful evasion also reveals a clear robustness hierarchy: SEAL, SFW, WIND, and PRC can be defeated with weak perturbation ($\lambda \leq 0.15$), RingID and ROBIN require a moderate attack strength ($\lambda = 0.30$), while Tree-Ring, GaussMarker, and Gaussian Shading are the most robust and require stronger perturbation ($\lambda \geq 0.50$). Nevertheless, SHIFT still achieves near-complete evasion on these more robust methods, highlighting that it is not only stronger than existing attacks, but also more stable and broadly applicable across diverse watermarking paradigms.Qualitative results across all nine watermarking methods are provided in Appendix~\ref{app:experiments} (Figure~\ref{fig:qualitative_all_methods} and Figure~\ref{fig:robustness_compare}).

\begin{figure*}[t]
    \centering
    \includegraphics[width=\textwidth]{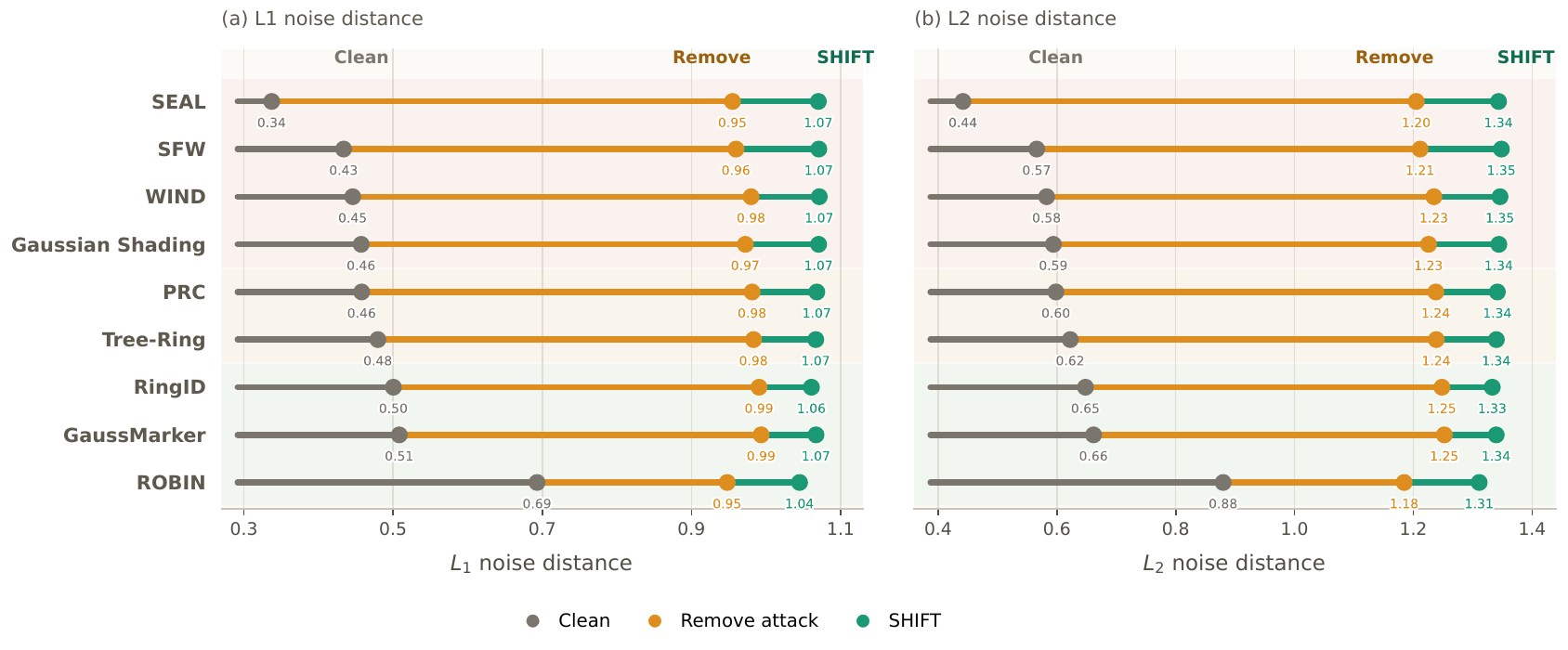}
    \caption{Comparison of $L_1$ and $L_2$ noise distances across nine watermarking methods.}
    \label{fig:noise}
\end{figure*}

\subsection{Trajectory Decoupling Analysis}

To verify whether SHIFT truly achieves decoupling from the original watermark trajectory, we further examine how far the attacked samples deviate from the original watermark-carrying trajectory in latent space. Specifically, for each watermarked image, we perform DDIM inversion on the outputs under the clean setting, removing attack, and SHIFT to recover the corresponding initial noise, and compute the $L_1$ and $L_2$ distances between the recovered noise and the original watermark-carrying noise. Here, the noise distance is not the ultimate object of interest, but rather an empirical proxy for the degree of trajectory decoupling: a larger distance indicates that the attacked sample has moved further away from the original watermark trajectory in latent space, making it harder for the verifier to map the image back to the embedded watermark signal. Figure~\ref{fig:noise} presents the $L_1$ and $L_2$ distances of nine watermarking methods under these three settings.

As shown in Figure~\ref{fig:noise}, SHIFT consistently produces the largest and most stable trajectory displacement across all methods, with highly aligned trends under both $L_1$ and $L_2$. Compared with the relatively small and method-dependent baseline distances under the clean setting, SHIFT pushes the recovered noise of all methods into a markedly higher and more concentrated range. For example, under the clean setting, the $L_1$ distances vary substantially across methods, ranging from 0.3373 for SEAL to 0.6930 for ROBIN, whereas under the strongest attack strength ($\lambda=0.7$), the post-attack distances largely converge to $L_1 \approx 1.05$--1.07 and $L_2 \approx 1.31$--1.34. In contrast, although removing attack also increases the distance, its effect is generally weaker and exhibits noticeably larger variation across methods. Notably, even for ROBIN, which already has a relatively large baseline distance under the clean setting, SHIFT still drives the recovered noise further toward a regime closer to random noise. Overall, these results provide direct empirical support for our central claim: SHIFT succeeds not by applying crude perturbations to image pixels, but by systematically disrupting trajectory consistency through stochastic resampling, thereby effectively decoupling the attacked sample from the original watermark trajectory. The full noise-distance 
curves as a function of $\lambda$ for all nine methods are shown 
in Appendix~\ref{app:noise_distance} 
(Figure~\ref{fig:noise_distance}).

% \begin{table}[t]
% \centering
% \caption{\textbf{Image quality comparison across attack methods.}}
% \label{tab:attack_compare}
% \vspace{1mm}
% \setlength{\tabcolsep}{6pt}
% \renewcommand{\arraystretch}{1.15}
% \small
% \begin{tabular}{lcc}
% \toprule
% \textbf{Method Group} & \textbf{CLIP score $\uparrow$} & \textbf{FID $\downarrow$} \\
% \midrule
% Black-box attack & 31.877 & 106.831 \\
% Forgery attack   & 31.356 & 96.431 \\
% Our attack       & 32.194 & 73.471 \\
% \bottomrule
% \end{tabular}
% \vspace{-2mm}
% \end{table}

\subsection{Image Quality Comparison}

\begin{figure*}[t]
\centering
\renewcommand{\arraystretch}{1.0}

\newlength{\methodcol}
\newlength{\imgcol}
\setlength{\methodcol}{0.08\textwidth}
\setlength{\imgcol}{0.23\textwidth}

\begin{minipage}[c]{\methodcol}\centering\end{minipage}%
\begin{minipage}[c]{\imgcol}\centering \scriptsize Clean \end{minipage}%
\begin{minipage}[c]{\imgcol}\centering \scriptsize SHIFT \end{minipage}%
\begin{minipage}[c]{\imgcol}\centering \scriptsize Black-Box Attack \end{minipage}%
\begin{minipage}[c]{\imgcol}\centering \scriptsize Removing Attack \end{minipage}

\vspace{1.5mm}

\begin{minipage}[c]{\methodcol}\centering
\small \textbf{PRC}
\end{minipage}%
\begin{minipage}[c]{\imgcol}\centering
\zoomfig{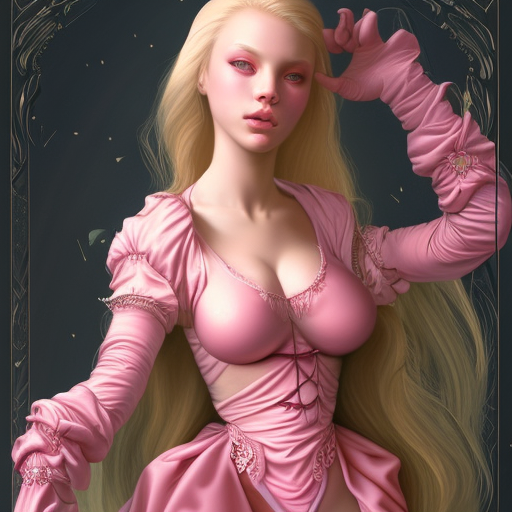}{2.55}{0.8}{0.68\linewidth}{1.25cm}
\end{minipage}%
\begin{minipage}[c]{\imgcol}\centering
\zoomfig{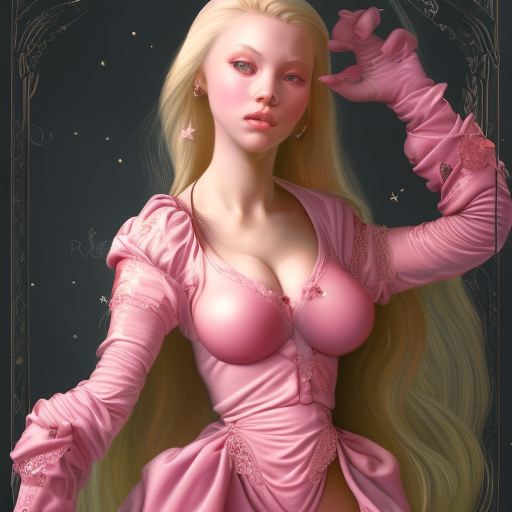}{2.55}{0.8}{0.68\linewidth}{1.25cm}
\end{minipage}%
\begin{minipage}[c]{\imgcol}\centering
\zoomfig{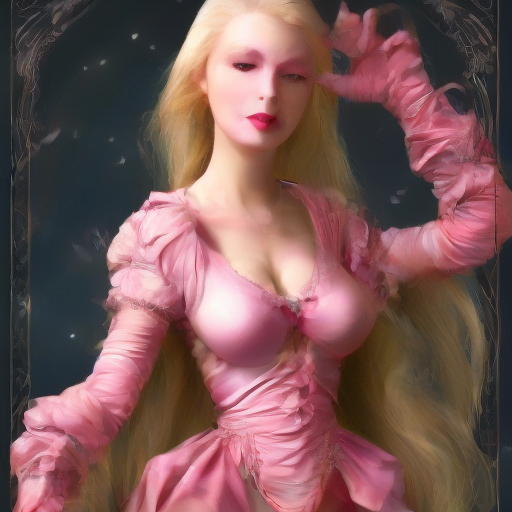}{2.55}{0.8}{0.68\linewidth}{1.25cm}
\end{minipage}%
\begin{minipage}[c]{\imgcol}\centering
\zoomfig{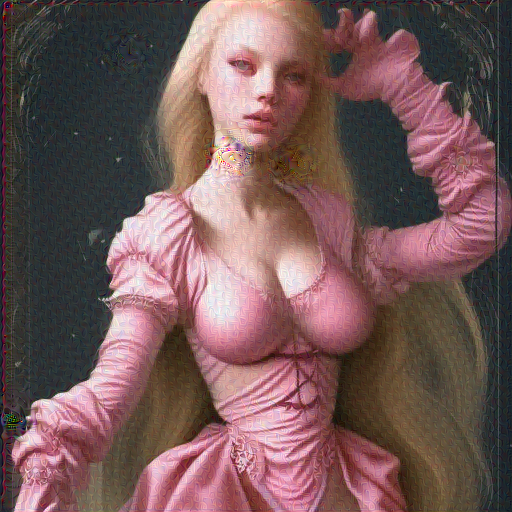}{2.55}{0.8}{0.68\linewidth}{1.25cm}
\end{minipage}

\vspace{1.5mm}

\begin{minipage}[c]{\methodcol}\centering
\small \textbf{ROBIN}
\end{minipage}%
\begin{minipage}[c]{\imgcol}\centering
\zoomfig{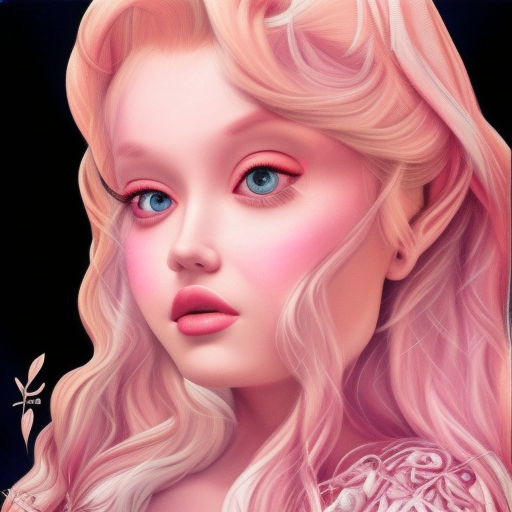}{0.72}{0.62}{0.68\linewidth}{1.25cm}
\end{minipage}%
\begin{minipage}[c]{\imgcol}\centering
\zoomfig{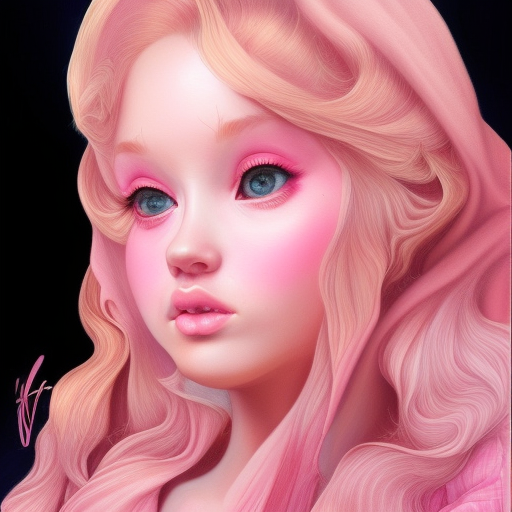}{0.72}{0.62}{0.68\linewidth}{1.25cm}
\end{minipage}%
\begin{minipage}[c]{\imgcol}\centering
\zoomfig{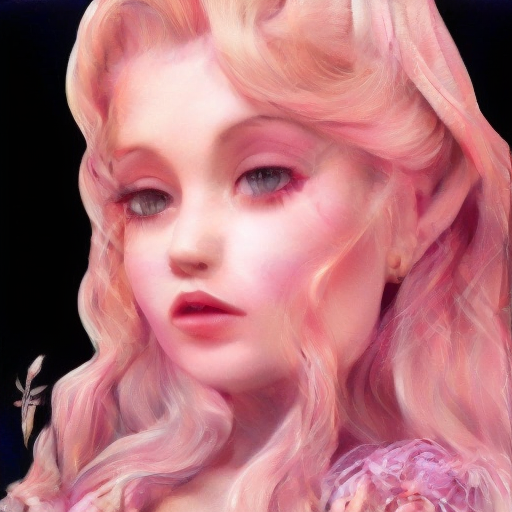}{0.72}{0.62}{0.68\linewidth}{1.25cm}
\end{minipage}%
\begin{minipage}[c]{\imgcol}\centering
\zoomfig{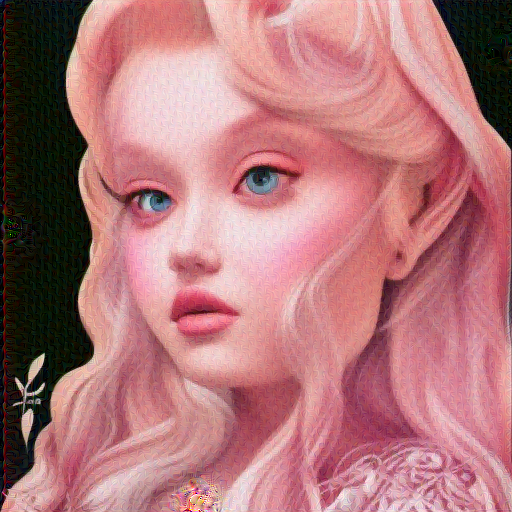}{0.72}{0.62}{0.68\linewidth}{1.25cm}
\end{minipage}

\vspace{1.5mm}

\begin{minipage}[c]{\methodcol}\centering
\small \textbf{SEAL}
\end{minipage}%
\begin{minipage}[c]{\imgcol}\centering
\zoomfig{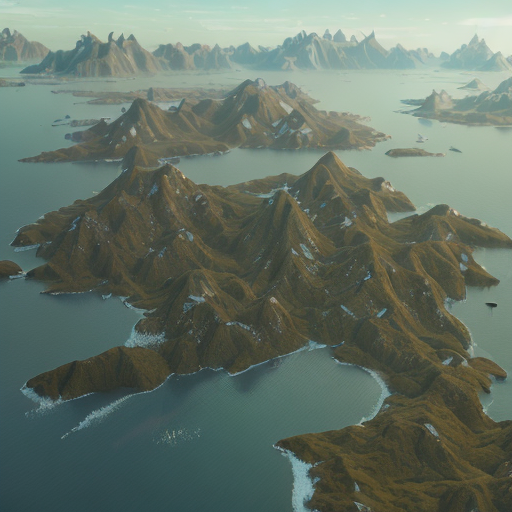}{1.4}{0.8}{0.68\linewidth}{1.25cm}
\end{minipage}%
\begin{minipage}[c]{\imgcol}\centering
\zoomfig{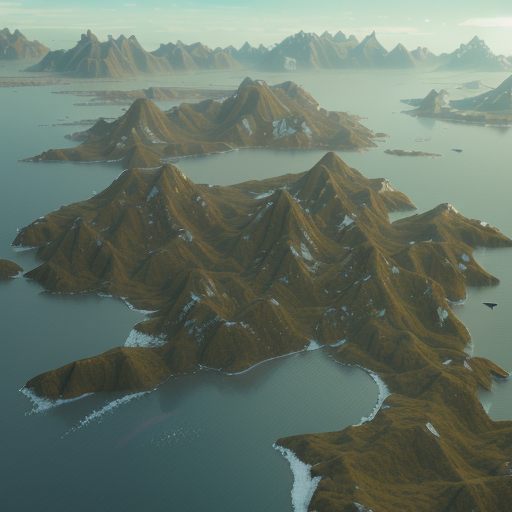}{1.4}{0.8}{0.68\linewidth}{1.25cm}
\end{minipage}%
\begin{minipage}[c]{\imgcol}\centering
\zoomfig{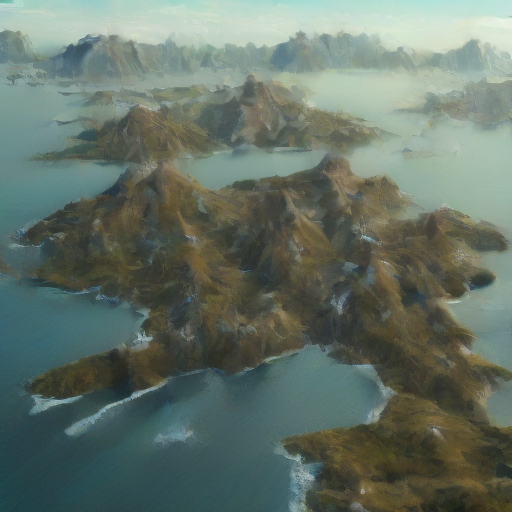}{1.4}{0.8}{0.68\linewidth}{1.25cm}
\end{minipage}%
\begin{minipage}[c]{\imgcol}\centering
\zoomfig{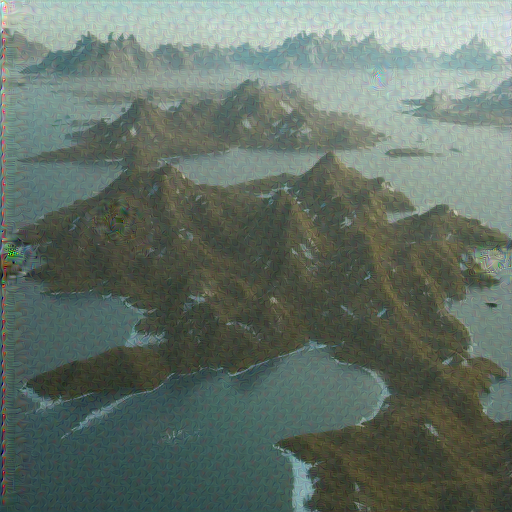}{1.4}{0.8}{0.68\linewidth}{1.25cm}
\end{minipage}

\caption{\textbf{Qualitative comparison with different attack.}}
\label{fig:qualitative_zoom_compare}
\end{figure*}

To verify whether SHIFT preserves image quality while removing watermarks, we compare SHIFT with the black-box attack and the latent-noise removing attack in terms of semantic preservation and distributional realism. The results are summarized in Table~\ref{tab:attack_compare}. SHIFT achieves the highest CLIP score (32.194), indicating that the attacked images maintain the strongest semantic consistency with the original generation prompts. Meanwhile, SHIFT attains an FID of 73.471, representing a relative reduction of 31.2\% and 23.8\% compared to the black-box attack (106.831) and the latent-noise removing attack (96.431), respectively, indicating that its generated images are closest to the real image distribution.

\begin{wraptable}{r}{0.45\textwidth}
\vspace{-4mm}
\centering
\caption{\textbf{Image quality comparison across attack methods.}}
\label{tab:attack_compare}
\vspace{1mm}
\setlength{\tabcolsep}{6pt}
\renewcommand{\arraystretch}{1.15}
\small
\begin{tabular}{lcc}
\toprule
\textbf{Method Group} & \textbf{CLIP score $\uparrow$} & \textbf{FID $\downarrow$} \\
\midrule
Black-box attack & 31.877 & 106.831 \\
Removing attack   & 31.356 & 96.431 \\
Our attack       & \textbf{32.194} & \textbf{73.471} \\
\bottomrule
\end{tabular}
\vspace{-2mm}
\end{wraptable}

This advantage stems from SHIFT's trajectory-level operation mechanism. The stochastic ancestral sampler, guided by the pretrained score function during the reverse denoising process, naturally maintains high visual realism in the reconstructed samples, thereby preserving image quality while disrupting the watermark trajectory. In contrast, the adversarial perturbations introduced by iterative optimization-based attacks degrade visual realism, resulting in significantly higher FID. Figure~\ref{fig:qualitative_zoom_compare} presents a visual comparison of the three attack methods at comparable evasion strength, where SHIFT produces images with noticeably better detail preservation and overall realism. Combined with the results in Table~\ref{tab:attack_compare}, SHIFT achieves near-complete watermark evasion while producing images of higher quality than existing attack methods.

%% file: 6_conclusion.tex
\section{Conclusion and Future Work}\label{sec:conclusion}

In this paper, we identify a critical vulnerability shared by diffusion-based watermarking schemes: their reliance on trajectory consistency between the embedding and verification stages. Exploiting this vulnerability, we propose SHIFT, a training-free, watermark-agnostic attack framework that erases trajectory-specific information through partial forward diffusion and redirects generation onto an entirely new trajectory via stochastic reverse resampling. Extensive experiments on nine representative watermarking methods spanning three paradigms demonstrate that SHIFT achieves 95\%--100\% attack success rates while maintaining the best image quality among all compared attacks, significantly outperforming existing black-box attack and latent-noise removing approaches. A complementary Wasserstein distance analysis confirms that the noise recovered after the attack is approximately independent of the original watermark-carrying noise, providing formal justification for the observed empirical effectiveness.

One limitation is that aggressive re-noising may degrade fine semantic details against particularly robust watermarking methods. This trade-off, however, points to several promising directions for future work. First, it would be valuable to design provenance mechanisms whose verification does not hinge on reconstructing a single diffusion trajectory but instead remains reliable under stochastic trajectory deflection. Second, extending this analysis beyond image diffusion to video diffusion watermarking~\cite{huvideo2026shield,hu2025videomark} presents additional challenges, as temporal coherence and cross-frame consistency impose structural constraints that may both complicate and inform attack strategies. More broadly, the rapid progress in scale-wise autoregressive and hybrid flow-based generation~\cite{tian2024visual,ren2025flowar,cao2026computational,hui2025arflow} raises the question of whether analogous trajectory-level dependencies emerge in generative pipelines that interleave autoregressive prediction with flow matching dynamics. We hope that SHIFT motivates a broader security perspective on AI content provenance, encompassing not only diffusion-based systems but also the growing family of trajectory-structured generative models.

%% file: 10_app.tex
\section{Theoretical Analysis}

We now formalize the trajectory-decoupling effect of SHIFT at an arbitrary attack depth $t_\lambda = \lceil\lambda T\rceil$, and specialize to the terminal step $t_\lambda = T$ to obtain the random-baseline distance.

Throughout the analysis, all expectations are taken over the watermarking randomness $\boldsymbol{\epsilon}^w$ and the fresh attack randomness $(\boldsymbol{\epsilon}, \boldsymbol{\xi}_{1:T})$.

\begin{assumption}[Regularity]
\label{asm:shift_regular}
We assume the followings:
\begin{enumerate}[label=(A\arabic*),nosep]
    \item the encoded latent $\mathbf{z}_0 = \mathcal{E}(\mathbf{x}^w)$ is square-integrable:
    \begin{align*}
    \mathbb{E}\bigl[\|\mathbf{z}_0\|_2^2\bigr] < \infty;
    \end{align*}
    \item for each $t \in \{1,\dots,T\}$, the denoiser $\boldsymbol{\epsilon}_\theta(\cdot,t)$ is globally $L_t$-Lipschitz:
    \begin{align*}
    \|\boldsymbol{\epsilon}_\theta(\mathbf{u},t)-\boldsymbol{\epsilon}_\theta(\mathbf{v},t)\|_2
    \le L_t \|\mathbf{u}-\mathbf{v}\|_2,
    \qquad \forall \mathbf{u},\mathbf{v}\in\mathbb{R}^d;
    \end{align*}
    \item the composition
    \begin{align*}
    Q := \mathcal{I}_{\mathrm{DDIM}} \circ \mathcal{E} \circ \mathcal{D},
    \end{align*}
    where $\mathcal{D}$ decodes a latent to pixel space, $\mathcal{E}$ re-encodes the pixel image to a latent, and $\mathcal{I}_{\mathrm{DDIM}}$ denotes DDIM inversion in latent space under the attacker's public model, is globally $L_Q$-Lipschitz:
    \begin{align*}
    \|Q(\mathbf{u})-Q(\mathbf{v})\|_2 \le L_Q \|\mathbf{u}-\mathbf{v}\|_2,
    \qquad \forall \mathbf{u},\mathbf{v}\in\mathbb{R}^d;
    \end{align*}
    \item the fresh noises $\boldsymbol{\epsilon}\sim\mathcal{N}(\mathbf{0},\mathbf{I}_d)$ and
    $\boldsymbol{\xi}_1,\dots,\boldsymbol{\xi}_T\sim\mathcal{N}(\mathbf{0},\mathbf{I}_d)$ are independent of $\boldsymbol{\epsilon}^w$.
\end{enumerate}
\end{assumption}

We define some useful notations.
For each $t\in\{1,\dots,T\}$, define
\begin{align*}
A_t := \sqrt{\frac{\bar{\alpha}_{t-1}}{\bar{\alpha}_t}},
\qquad
B_t := \sqrt{1-\bar{\alpha}_{t-1}-\sigma_t^2}
-
\sqrt{\frac{\bar{\alpha}_{t-1}(1-\bar{\alpha}_t)}{\bar{\alpha}_t}},
\end{align*}
and let
\begin{align*}
\rho_t := A_t + |B_t|L_t,
\qquad
C_n := \prod_{s=1}^{n}\rho_s,
\qquad n=1,\dots,T.
\end{align*}

For fixed $\boldsymbol{\xi}_{1:n}$ and starting point $\mathbf{u}\in\mathbb{R}^d$, define the $n$-step reverse map $F_n(\mathbf{u};\boldsymbol{\xi}_{1:n})$ recursively by
\begin{align*}
\mathbf{z}^{(\mathbf{u})}_n = \mathbf{u},
\end{align*}
\begin{align*}
\mathbf{z}^{(\mathbf{u})}_{t-1}
=
A_t\,\mathbf{z}^{(\mathbf{u})}_t
+
B_t\,\boldsymbol{\epsilon}_\theta(\mathbf{z}^{(\mathbf{u})}_t,t)
+
\sigma_t\boldsymbol{\xi}_t,
\qquad t=n,n-1,\dots,1,
\end{align*}
and
\begin{align*}
F_n(\mathbf{u};\boldsymbol{\xi}_{1:n}) := \mathbf{z}^{(\mathbf{u})}_0.
\end{align*}
We further define the recovered-noise map
\begin{align*}
R_n(\mathbf{u};\boldsymbol{\xi}_{1:n})
:=
Q\bigl(F_n(\mathbf{u};\boldsymbol{\xi}_{1:n})\bigr).
\end{align*}

\begin{lemma}[One-step stability of the ancestral update]
\label{lem:shift_one_step}
For every $t\in\{1,\dots,T\}$, every $\mathbf{u},\mathbf{v}\in\mathbb{R}^d$, and every $\boldsymbol{\xi}\in\mathbb{R}^d$,
\begin{align*}
\bigl\|
(A_t\mathbf{u} + B_t\boldsymbol{\epsilon}_\theta(\mathbf{u},t) + \sigma_t\boldsymbol{\xi})
-
(A_t\mathbf{v} + B_t\boldsymbol{\epsilon}_\theta(\mathbf{v},t) + \sigma_t\boldsymbol{\xi})
\bigr\|_2
\le
\rho_t \|\mathbf{u}-\mathbf{v}\|_2.
\end{align*}
\end{lemma}
\begin{proof}
By cancellation of the common noise term $\sigma_t\boldsymbol{\xi}$,
\begin{align*}
&\bigl\|
A_t\mathbf{u} + B_t\boldsymbol{\epsilon}_\theta(\mathbf{u},t)
-
A_t\mathbf{v} - B_t\boldsymbol{\epsilon}_\theta(\mathbf{v},t)
\bigr\|_2 \\
&\qquad\le
A_t\|\mathbf{u}-\mathbf{v}\|_2
+
|B_t|\,
\|\boldsymbol{\epsilon}_\theta(\mathbf{u},t)-\boldsymbol{\epsilon}_\theta(\mathbf{v},t)\|_2.
\end{align*}
Applying Assumption~\ref{asm:shift_regular}(A2) gives
\begin{align*}
\|\boldsymbol{\epsilon}_\theta(\mathbf{u},t)-\boldsymbol{\epsilon}_\theta(\mathbf{v},t)\|_2
\le L_t\|\mathbf{u}-\mathbf{v}\|_2.
\end{align*}
Hence, we have
\begin{align*}
&~  \bigl\|
A_t\mathbf{u} + B_t\boldsymbol{\epsilon}_\theta(\mathbf{u},t) + \sigma_t\boldsymbol{\xi}
-
A_t\mathbf{v} - B_t\boldsymbol{\epsilon}_\theta(\mathbf{v},t) - \sigma_t\boldsymbol{\xi}
\bigr\|_2
\\
\le &~
(A_t+|B_t|L_t)\|\mathbf{u}-\mathbf{v}\|_2
\\
= &~
\rho_t\|\mathbf{u}-\mathbf{v}\|_2. 
\end{align*}
\end{proof}

\begin{lemma}[Multi-step stability]
\label{lem:shift_multistep}
For every $n\in\{1,\dots,T\}$, every $\mathbf{u},\mathbf{v}\in\mathbb{R}^d$, and every realization of $\boldsymbol{\xi}_{1:n}$,
\begin{align*}
\|F_n(\mathbf{u};\boldsymbol{\xi}_{1:n}) - F_n(\mathbf{v};\boldsymbol{\xi}_{1:n})\|_2
\le
C_n \|\mathbf{u}-\mathbf{v}\|_2.
\end{align*}
\end{lemma}

\begin{proof}
Let $\mathbf{z}^{(\mathbf{u})}_t$ and $\mathbf{z}^{(\mathbf{v})}_t$ denote the two reverse trajectories started from $\mathbf{u}$ and $\mathbf{v}$ and driven by the same noises $\boldsymbol{\xi}_{1:n}$.
By Lemma~\ref{lem:shift_one_step},
\begin{align*}
\|\mathbf{z}^{(\mathbf{u})}_{t-1} - \mathbf{z}^{(\mathbf{v})}_{t-1}\|_2
\le
\rho_t
\|\mathbf{z}^{(\mathbf{u})}_{t} - \mathbf{z}^{(\mathbf{v})}_{t}\|_2,
\qquad t=n,n-1,\dots,1.
\end{align*}
Iterating this inequality from $t=n$ down to $t=1$ yields
\begin{align*}
\|\mathbf{z}^{(\mathbf{u})}_{0} - \mathbf{z}^{(\mathbf{v})}_{0}\|_2
\le
\Bigl(\prod_{s=1}^{n}\rho_s\Bigr)
\|\mathbf{u}-\mathbf{v}\|_2
=
C_n\|\mathbf{u}-\mathbf{v}\|_2.
\end{align*}
Since $F_n(\mathbf{u};\boldsymbol{\xi}_{1:n})=\mathbf{z}^{(\mathbf{u})}_0$ and
$F_n(\mathbf{v};\boldsymbol{\xi}_{1:n})=\mathbf{z}^{(\mathbf{v})}_0$, the claim follows.
\end{proof}

\begin{lemma}[Stability of the recovered-noise map]
\label{lem:shift_recovered_stability}
For every $n\in\{1,\dots,T\}$, every $\mathbf{u},\mathbf{v}\in\mathbb{R}^d$, and every realization of $\boldsymbol{\xi}_{1:n}$,
\begin{align*}
\|R_n(\mathbf{u};\boldsymbol{\xi}_{1:n}) - R_n(\mathbf{v};\boldsymbol{\xi}_{1:n})\|_2
\le
L_Q C_n \|\mathbf{u}-\mathbf{v}\|_2.
\end{align*}
\end{lemma}

\begin{proof}
By definition,
\begin{align*}
R_n(\mathbf{u};\boldsymbol{\xi}_{1:n})
=
Q\bigl(F_n(\mathbf{u};\boldsymbol{\xi}_{1:n})\bigr),
\qquad
R_n(\mathbf{v};\boldsymbol{\xi}_{1:n})
=
Q\bigl(F_n(\mathbf{v};\boldsymbol{\xi}_{1:n})\bigr).
\end{align*}
Assumption~\ref{asm:shift_regular}(A3) implies
\begin{align*}
\|R_n(\mathbf{u};\boldsymbol{\xi}_{1:n}) - R_n(\mathbf{v};\boldsymbol{\xi}_{1:n})\|_2
\le
L_Q
\|F_n(\mathbf{u};\boldsymbol{\xi}_{1:n}) - F_n(\mathbf{v};\boldsymbol{\xi}_{1:n})\|_2.
\end{align*}
Applying Lemma~\ref{lem:shift_multistep} gives
\begin{align*}
\|R_n(\mathbf{u};\boldsymbol{\xi}_{1:n}) - R_n(\mathbf{v};\boldsymbol{\xi}_{1:n})\|_2
\le
L_Q C_n \|\mathbf{u}-\mathbf{v}\|_2. 
\end{align*}
\end{proof}

\begin{lemma}[Decoupling from the source latent]
\label{lem:shift_terminal_decoupling}
For each $n\in\{1,\dots,T\}$, define
\begin{align*}
\hat{\boldsymbol{\epsilon}}_n^a
:=
R_n\bigl(\sqrt{\bar{\alpha}_n}\,\mathbf{z}_0
+
\sqrt{1-\bar{\alpha}_n}\,\boldsymbol{\epsilon};\boldsymbol{\xi}_{1:n}\bigr),
\qquad
\tilde{\boldsymbol{\epsilon}}_n
:=
R_n\bigl(\sqrt{1-\bar{\alpha}_n}\,\boldsymbol{\epsilon};\boldsymbol{\xi}_{1:n}\bigr).
\end{align*}
Then $\tilde{\boldsymbol{\epsilon}}_n$ is independent of $\boldsymbol{\epsilon}^w$, and
\begin{align*}
\mathbb{E}\bigl[
\|\hat{\boldsymbol{\epsilon}}_n^a - \tilde{\boldsymbol{\epsilon}}_n\|_2^2
\bigr]
\le
L_Q^2 C_n^2 \bar{\alpha}_n\,
\mathbb{E}\bigl[\|\mathbf{z}_0\|_2^2\bigr].
\end{align*}
\end{lemma}

\begin{proof}
By construction, $\tilde{\boldsymbol{\epsilon}}_n$ is a measurable function only of
$(\boldsymbol{\epsilon},\boldsymbol{\xi}_{1:n})$ and the public model, and Assumption~\ref{asm:shift_regular}(A4) states that $(\boldsymbol{\epsilon},\boldsymbol{\xi}_{1:n})$ is independent of $\boldsymbol{\epsilon}^w$.
Hence $\tilde{\boldsymbol{\epsilon}}_n$ is independent of $\boldsymbol{\epsilon}^w$.

For the norm bound, apply Lemma~\ref{lem:shift_recovered_stability} with
\begin{align*}
\mathbf{u}
=
\sqrt{\bar{\alpha}_n}\,\mathbf{z}_0 + \sqrt{1-\bar{\alpha}_n}\,\boldsymbol{\epsilon},
\qquad
\mathbf{v}
=
\sqrt{1-\bar{\alpha}_n}\,\boldsymbol{\epsilon}.
\end{align*}
Then
$\mathbf{u}-\mathbf{v}
=
\sqrt{\bar{\alpha}_n}\,\mathbf{z}_0$,
and therefore
\begin{align*}
\|\hat{\boldsymbol{\epsilon}}_n^a - \tilde{\boldsymbol{\epsilon}}_n\|_2
\le
L_Q C_n \sqrt{\bar{\alpha}_n}\,\|\mathbf{z}_0\|_2.
\end{align*}
Squaring both sides and taking expectations yields the claim.
\end{proof}

\begin{lemma}[Distance to the independent product law]
\label{lem:shift_product_law}
Under Assumption~\ref{asm:shift_regular}, for every $n\in\{1,\dots,T\}$,
\begin{align*}
W_2\left(
\mathcal{L}(\hat{\boldsymbol{\epsilon}}_n^a,\boldsymbol{\epsilon}^w),
\mathcal{L}(\tilde{\boldsymbol{\epsilon}}_n)\otimes \mathcal{L}(\boldsymbol{\epsilon}^w)
\right)
\le
L_Q C_n \sqrt{\bar{\alpha}_n\,\mathbb{E}\bigl[\|\mathbf{z}_0\|_2^2\bigr]}.
\end{align*}
Consequently,
\begin{align*}
W_2\left(
\mathcal{L}(\hat{\boldsymbol{\epsilon}}_n^a,\boldsymbol{\epsilon}^w),
\mathcal{L}(\hat{\boldsymbol{\epsilon}}_n^a)\otimes \mathcal{L}(\boldsymbol{\epsilon}^w)
\right)
\le
2L_Q C_n \sqrt{\bar{\alpha}_n\,\mathbb{E}\bigl[\|\mathbf{z}_0\|_2^2\bigr]}.
\end{align*}
\end{lemma}

\begin{proof}
Because $\tilde{\boldsymbol{\epsilon}}_n$ is independent of $\boldsymbol{\epsilon}^w$ by Lemma~\ref{lem:shift_terminal_decoupling},
\begin{align*}
\mathcal{L}(\tilde{\boldsymbol{\epsilon}}_n,\boldsymbol{\epsilon}^w)
=
\mathcal{L}(\tilde{\boldsymbol{\epsilon}}_n)\otimes \mathcal{L}(\boldsymbol{\epsilon}^w).
\end{align*}
Now consider the coupling on the underlying probability space given by
\begin{align*}
\bigl((\hat{\boldsymbol{\epsilon}}_n^a,\boldsymbol{\epsilon}^w),(\tilde{\boldsymbol{\epsilon}}_n,\boldsymbol{\epsilon}^w)\bigr).
\end{align*}
Its second marginal is precisely
$\mathcal{L}(\tilde{\boldsymbol{\epsilon}}_n)\otimes \mathcal{L}(\boldsymbol{\epsilon}^w)$.
Hence
\begin{align*}
W_2^2\left(
\mathcal{L}(\hat{\boldsymbol{\epsilon}}_n^a,\boldsymbol{\epsilon}^w),
\mathcal{L}(\tilde{\boldsymbol{\epsilon}}_n)\otimes \mathcal{L}(\boldsymbol{\epsilon}^w)
\right)
&\le
\mathbb{E}\bigl[
\|(\hat{\boldsymbol{\epsilon}}_n^a,\boldsymbol{\epsilon}^w)-(\tilde{\boldsymbol{\epsilon}}_n,\boldsymbol{\epsilon}^w)\|_2^2
\bigr] \\
&=
\mathbb{E}\bigl[
\|\hat{\boldsymbol{\epsilon}}_n^a-\tilde{\boldsymbol{\epsilon}}_n\|_2^2
\bigr].
\end{align*}
Applying Lemma~\ref{lem:shift_terminal_decoupling} yields the first bound.

For the second bound, use the triangle inequality for $W_2$:
\begin{align*}
 &W_2\left(
\mathcal{L}(\hat{\boldsymbol{\epsilon}}_n^a,\boldsymbol{\epsilon}^w),
\mathcal{L}(\hat{\boldsymbol{\epsilon}}_n^a)\otimes \mathcal{L}(\boldsymbol{\epsilon}^w)
\right) \le 
W_2\left(
\mathcal{L}(\hat{\boldsymbol{\epsilon}}_n^a,\boldsymbol{\epsilon}^w),
\mathcal{L}(\tilde{\boldsymbol{\epsilon}}_n)\otimes \mathcal{L}(\boldsymbol{\epsilon}^w)
\right) \\ &\qquad\qquad\qquad\qquad\qquad+ 
W_2\left(
\mathcal{L}(\tilde{\boldsymbol{\epsilon}}_n)\otimes \mathcal{L}(\boldsymbol{\epsilon}^w),
\mathcal{L}(\hat{\boldsymbol{\epsilon}}_n^a)\otimes \mathcal{L}(\boldsymbol{\epsilon}^w)
\right).
\end{align*}
The second term is bounded in the same way as the first, again by
$L_Q C_n \sqrt{\bar{\alpha}_n\,\mathbb{E}\bigl[\|\mathbf{z}_0\|_2^2\bigr]}$.
Adding the two bounds proves the claim.
\end{proof}

\begin{theorem}[Trajectory decoupling at arbitrary attack depth]
\label{prop:noise_distance_formal}
Under Assumption~\ref{asm:shift_regular}, for every $\lambda\in(0,1]$ with $t_\lambda = \lceil\lambda T\rceil$:
\begin{enumerate}[label=(\roman*),nosep]
    \item It holds that \begin{align*}\mathbb{E}\bigl[\|\hat{\boldsymbol{\epsilon}}_{t_\lambda}^a - \tilde{\boldsymbol{\epsilon}}_{t_\lambda}\|_2^2\bigr]
    \le \Delta_{t_\lambda},~~~\text{where}~~~
    \Delta_{t_\lambda} := L_Q^2 C_{t_\lambda}^2\, \bar{\alpha}_{t_\lambda}\,\mathbb{E}\bigl[\|\mathbf{z}_0\|_2^2\bigr];
    \end{align*}
    \item the joint law of the attacked recovered noise and the watermark noise satisfies
    \begin{align*}
    W_2\left(
    \mathcal{L}(\hat{\boldsymbol{\epsilon}}_{t_\lambda}^a,\boldsymbol{\epsilon}^w),
    \mathcal{L}(\hat{\boldsymbol{\epsilon}}_{t_\lambda}^a)\otimes \mathcal{L}(\boldsymbol{\epsilon}^w)
    \right)
    \le
    2\sqrt{\Delta_{t_\lambda}}.
    \end{align*}
\end{enumerate}
\end{theorem}

\begin{proof}
Part (i) is Lemma~\ref{lem:shift_terminal_decoupling} applied with $n = t_\lambda$.
Part (ii) is Lemma~\ref{lem:shift_product_law} applied with $n = t_\lambda$.
\end{proof}

\begin{remark}[Interpretation of $\Delta_{t_\lambda}$]
\label{rmk:delta_interpretation}
Note that the bound $\Delta_{t_\lambda} = L_Q^2 C_{t_\lambda}^2\,\bar{\alpha}_{t_\lambda}\,\mathbb{E}[\|\mathbf{z}_0\|_2^2]$ reflects two competing effects.
On one hand, $\bar{\alpha}_{t_\lambda}$ decreases with $t_\lambda$, corresponding to stronger noise erasure of the source signal.
On the other hand, $C_{t_\lambda} = \prod_{s=1}^{t_\lambda}\rho_s$ is a product of per-step contraction factors $\rho_s = A_s + |B_s|L_s$; under typical noise schedules and denoiser regularity, $\rho_s \ge 1$ for most steps, so $C_{t_\lambda}$ generally grows with $t_\lambda$, reflecting error accumulation over more reverse steps.
Whether the overall bound $\Delta_{t_\lambda}$ decreases with $t_\lambda$ depends on the rate at which $\bar{\alpha}_{t_\lambda} \to 0$ dominates the growth of $C_{t_\lambda}^2$, a balance that is schedule- and model-dependent.
\end{remark}

\begin{assumption}[Terminal prior consistency]
\label{asm:shift_prior}
The terminal baseline recovered noise
$
\tilde{\boldsymbol{\epsilon}}_T
=
R_T\bigl(\sqrt{1-\bar{\alpha}_T}\,\boldsymbol{\epsilon};\boldsymbol{\xi}_{1:T}\bigr)
$
is distributed as $\mathcal{N}(\mathbf{0},\mathbf{I}_d)$.
Moreover, the watermark-carrying initial noise satisfies $\boldsymbol{\epsilon}^w$ is also distributed as $\mathcal{N}(\mathbf{0},\mathbf{I}_d)$.
\end{assumption}

\noindent
The condition on $\boldsymbol{\epsilon}^w$ is satisfied or closely approximated by most watermarking schemes: methods such as Gaussian Shading and PRC draw the initial noise from distributions that are marginally close to $\mathcal{N}(\mathbf{0},\mathbf{I}_d)$, and even schemes that impose structured patterns (e.g., Tree-Ring) preserve $\mathbb{E}[\|\boldsymbol{\epsilon}^w\|_2^2] \approx d$, which is the only property used in the proof.
The condition on $\tilde{\boldsymbol{\epsilon}}_T$ is an idealization.
It requires that starting from near-pure noise ($\bar{\alpha}_T \approx 0$), running $T$~steps of ancestral sampling, decoding to pixel space, re-encoding, and performing DDIM inversion yields a standard Gaussian.
This holds exactly when the diffusion model is perfectly calibrated and the VAE satisfies $\mathcal{E} \circ \mathcal{D} = \mathrm{id}$.
In practice, neither condition is exact, but modern latent diffusion models satisfy both to a close approximation, and our experimental noise-distance measurements (Table~2, $L_1@0.7 \approx 1.07$) confirm that the post-attack recovered noise is empirically indistinguishable from a random Gaussian draw.

\begin{corollary}[Random-baseline noise distance at the terminal step]
\label{cor:shift_2d}
Under Assumptions~\ref{asm:shift_regular} and~\ref{asm:shift_prior}, it holds that
\begin{align*}
\left|
\mathbb{E}\bigl[
\|\hat{\boldsymbol{\epsilon}}_T^a-\boldsymbol{\epsilon}^w\|_2^2
\bigr]
-
2d
\right|
\le
\Delta_T + 2\sqrt{2d\,\Delta_T},
\end{align*}
where $\Delta_T = \Delta_{t_\lambda}\big|_{t_\lambda=T}$ is the terminal-step specialization of Proposition~\ref{prop:noise_distance_formal}(i).
\end{corollary}

\begin{proof}
It suffices to bound
$\mathbb{E}\bigl[\|\hat{\boldsymbol{\epsilon}}_T^a-\boldsymbol{\epsilon}^w\|_2^2\bigr]$.
Using
\begin{align*}
\hat{\boldsymbol{\epsilon}}_T^a-\boldsymbol{\epsilon}^w
=
(\hat{\boldsymbol{\epsilon}}_T^a-\tilde{\boldsymbol{\epsilon}}_T)
+
(\tilde{\boldsymbol{\epsilon}}_T-\boldsymbol{\epsilon}^w),
\end{align*}
we obtain
\begin{align*}
&\mathbb{E}\bigl[
\|\hat{\boldsymbol{\epsilon}}_T^a-\boldsymbol{\epsilon}^w\|_2^2
\bigr]
-
\mathbb{E}\bigl[
\|\tilde{\boldsymbol{\epsilon}}_T-\boldsymbol{\epsilon}^w\|_2^2
\bigr] \\
&\qquad=
\mathbb{E}\bigl[
\|\hat{\boldsymbol{\epsilon}}_T^a-\tilde{\boldsymbol{\epsilon}}_T\|_2^2
\bigr]
+
2\mathbb{E}\bigl[
\langle
\hat{\boldsymbol{\epsilon}}_T^a-\tilde{\boldsymbol{\epsilon}}_T,\,
\tilde{\boldsymbol{\epsilon}}_T-\boldsymbol{\epsilon}^w
\rangle
\bigr].
\end{align*}
Therefore, by Cauchy--Schwarz,
\begin{align*}
&\left|
\mathbb{E}\bigl[
\|\hat{\boldsymbol{\epsilon}}_T^a-\boldsymbol{\epsilon}^w\|_2^2
\bigr]
-
\mathbb{E}\bigl[
\|\tilde{\boldsymbol{\epsilon}}_T-\boldsymbol{\epsilon}^w\|_2^2
\bigr]
\right| \\
&\qquad\le
\mathbb{E}\bigl[
\|\hat{\boldsymbol{\epsilon}}_T^a-\tilde{\boldsymbol{\epsilon}}_T\|_2^2
\bigr]
+
2
\Bigl(
\mathbb{E}\bigl[
\|\hat{\boldsymbol{\epsilon}}_T^a-\tilde{\boldsymbol{\epsilon}}_T\|_2^2
\bigr]
\Bigr)^{1/2}
\Bigl(
\mathbb{E}\bigl[
\|\tilde{\boldsymbol{\epsilon}}_T-\boldsymbol{\epsilon}^w\|_2^2
\bigr]
\Bigr)^{1/2}.
\end{align*}
By Proposition~\ref{prop:noise_distance_formal}(i) with $t_\lambda = T$,
\begin{align*}
\mathbb{E}\bigl[
\|\hat{\boldsymbol{\epsilon}}_T^a-\tilde{\boldsymbol{\epsilon}}_T\|_2^2
\bigr]
\le
\Delta_T.
\end{align*}
By Assumption~\ref{asm:shift_prior}, $\tilde{\boldsymbol{\epsilon}}_T$ and $\boldsymbol{\epsilon}^w$ are independent standard Gaussian vectors, so
\begin{align*}
\mathbb{E}\bigl[
\|\tilde{\boldsymbol{\epsilon}}_T-\boldsymbol{\epsilon}^w\|_2^2
\bigr]
=
\mathbb{E}\|\tilde{\boldsymbol{\epsilon}}_T\|_2^2
+
\mathbb{E}\|\boldsymbol{\epsilon}^w\|_2^2
=
d+d
=
2d.
\end{align*}
Hence
\begin{align*}
\left|
\mathbb{E}\bigl[
\|\hat{\boldsymbol{\epsilon}}_T^a-\boldsymbol{\epsilon}^w\|_2^2
\bigr]
-
2d
\right|
\le
\Delta_T + 2\sqrt{2d\,\Delta_T}. 
\end{align*}
\end{proof}

\section{Additional Experimental Results}
\label{app:experiments}

This appendix provides supplementary qualitative and quantitative 
results that complement the main experiments in Section~5.

\subsection{Qualitative Comparison Across All Watermarking Methods}
\label{app:qualitative_all}

Figure~\ref{fig:qualitative_all_methods} provides a comprehensive 
visual comparison of SHIFT's output quality across all nine evaluated 
watermarking methods at attack strength $\lambda=0.50$. For each 
method, the left column shows the original watermarked image and the 
right column shows the corresponding attacked image. Across all nine 
methods---spanning noise-space (TR, RI, PRC, WIND), frequency-domain 
(GS, GM, SFW), and optimization-based (ROBIN, SEAL) paradigms---SHIFT 
consistently preserves the subject identity, color palette, and 
compositional structure of the original. The attacked images are 
visually indistinguishable from the watermarked originals, confirming 
that trajectory deflection via stochastic resampling does not introduce 
perceptible pixel-space artifacts.

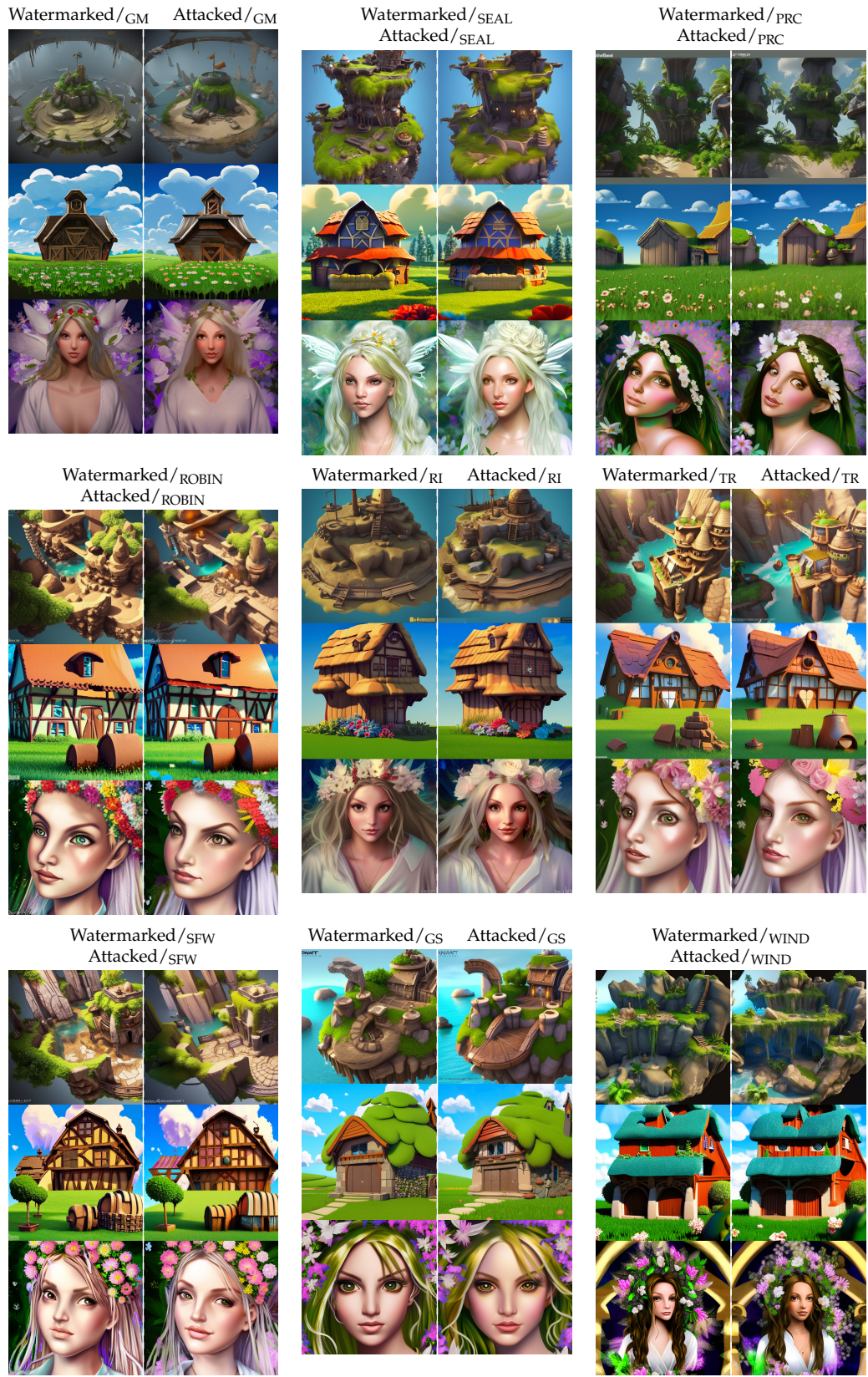
\begin{figure*}[!htp]
    \centering

    \newcommand{\figw}{0.26\textwidth}
\newcommand{\colgap}{0pt}
\newcommand{\rowgap}{0.4em}

\newcommand{\methodfig}[2]{
    \begin{subfigure}[t]{\figw}
        \centering
        {\scriptsize Watermarked/$_{\mathrm{#1}}$
        \hspace{0.8em}
        Attacked/$_{\mathrm{#1}}$\par}
        \vspace{1pt}
        \begin{tikzpicture}[baseline]
            \node[anchor=south west,inner sep=0] (img) at (0,0)
            {\includegraphics[width=\linewidth,trim=6 6 6 6,clip]{#2}};
            \draw[dashed, white, line width=0.6pt]
                ($(img.south west)!0.5!(img.south east)$) --
                ($(img.north west)!0.5!(img.north east)$);
        \end{tikzpicture}
    \end{subfigure}
}
    \methodfig{GM}{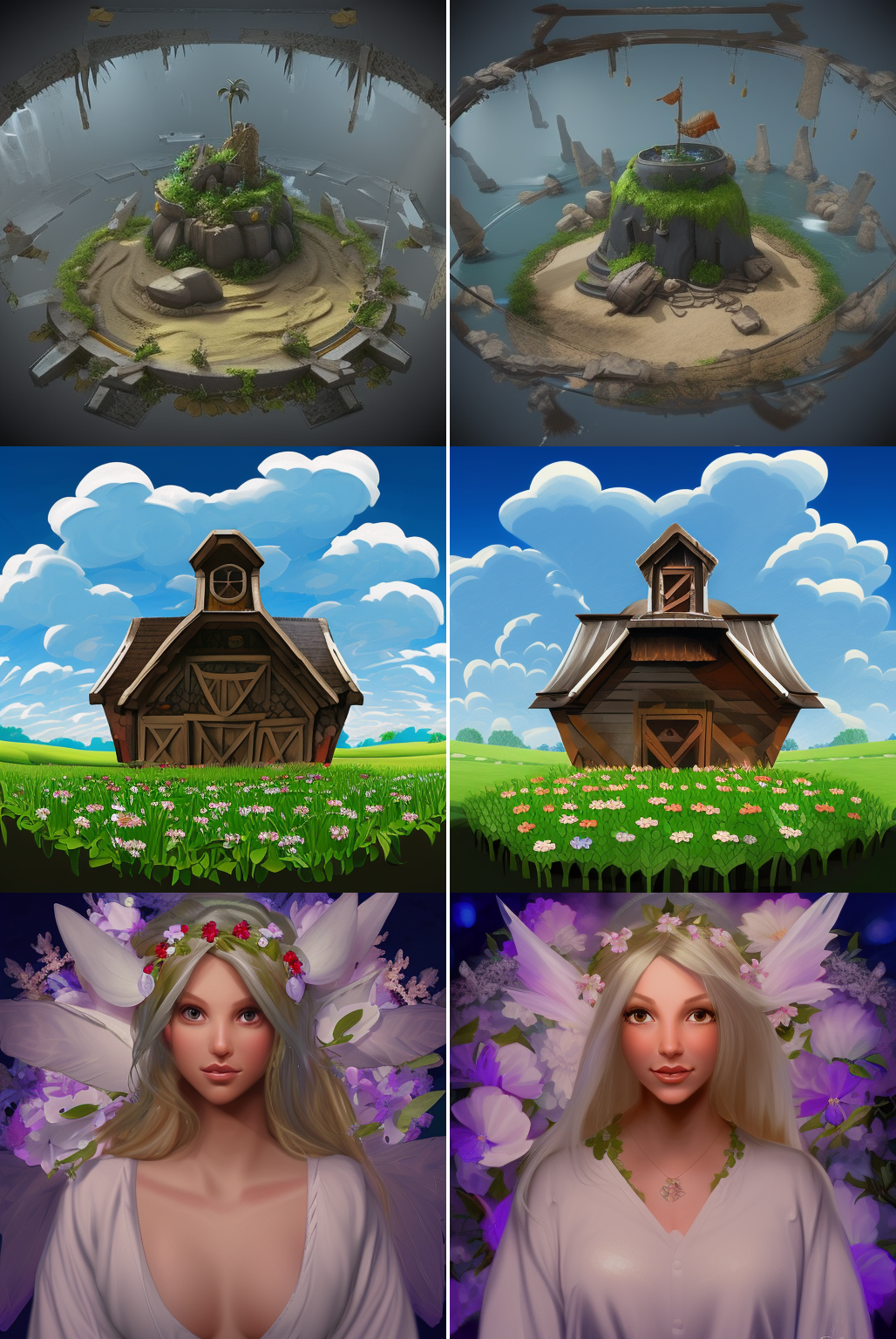}
    \hspace{\colgap}
    \methodfig{SEAL}{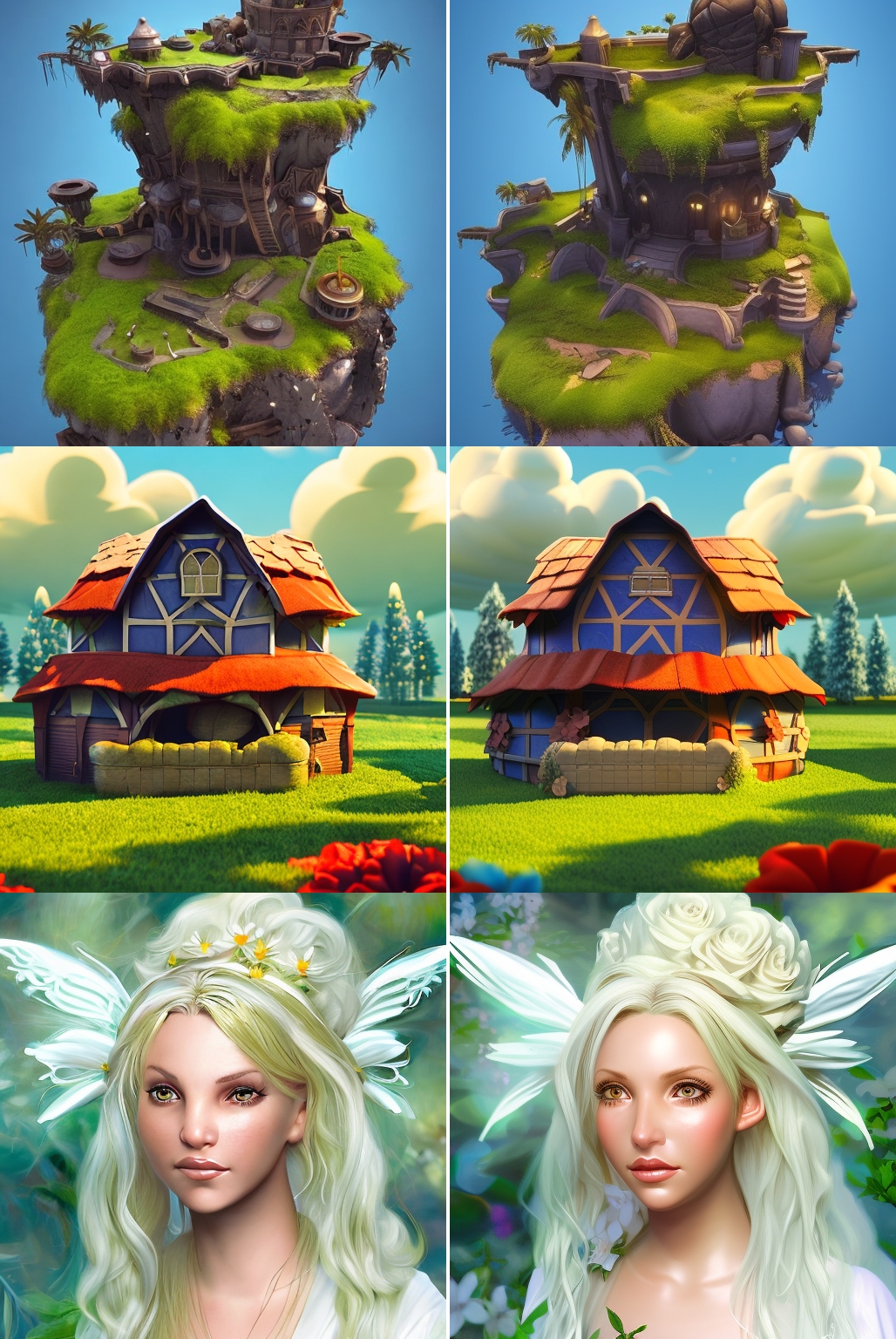}
    \hspace{\colgap}
    \methodfig{PRC}{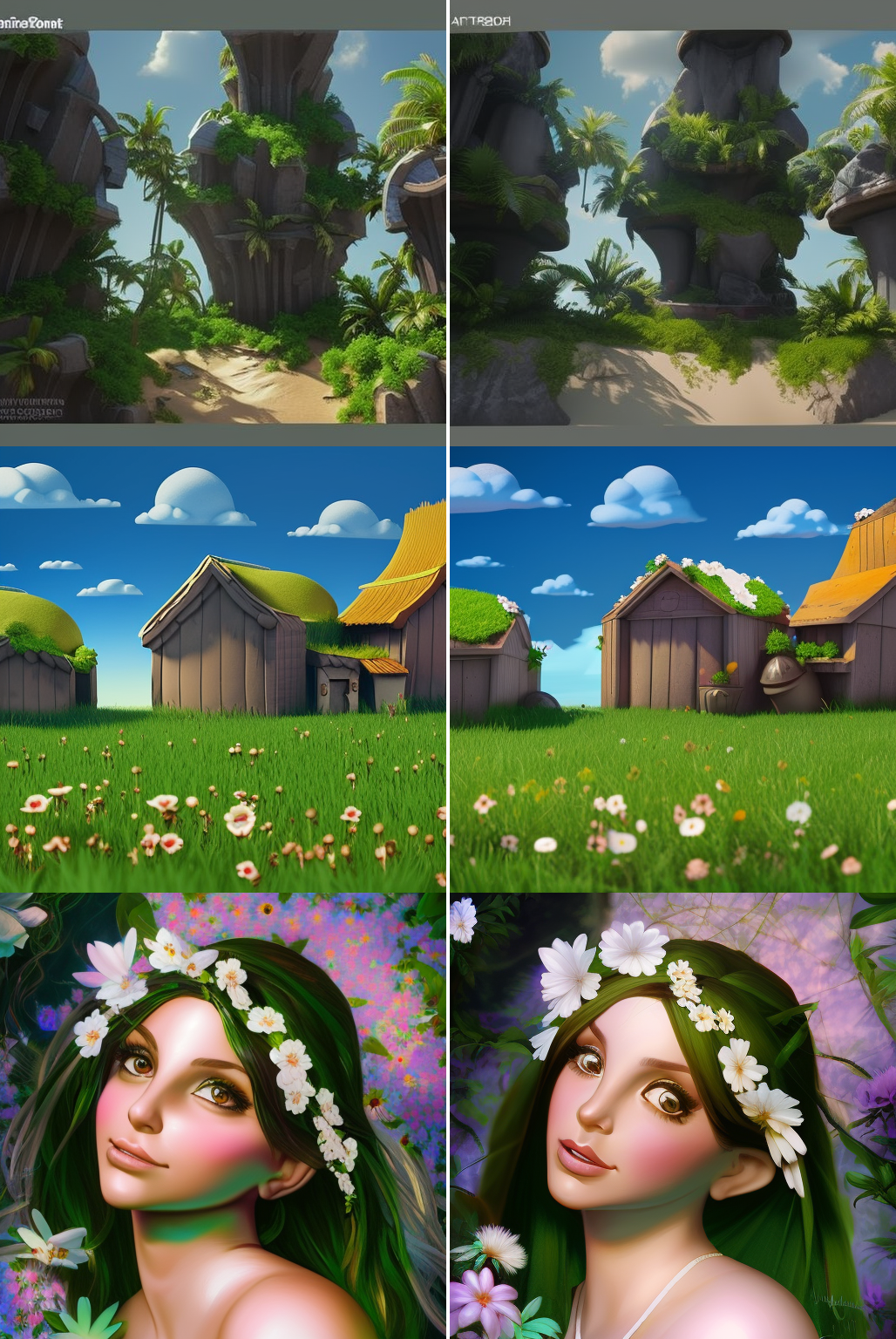}

    \vspace{\rowgap}

    \methodfig{ROBIN}{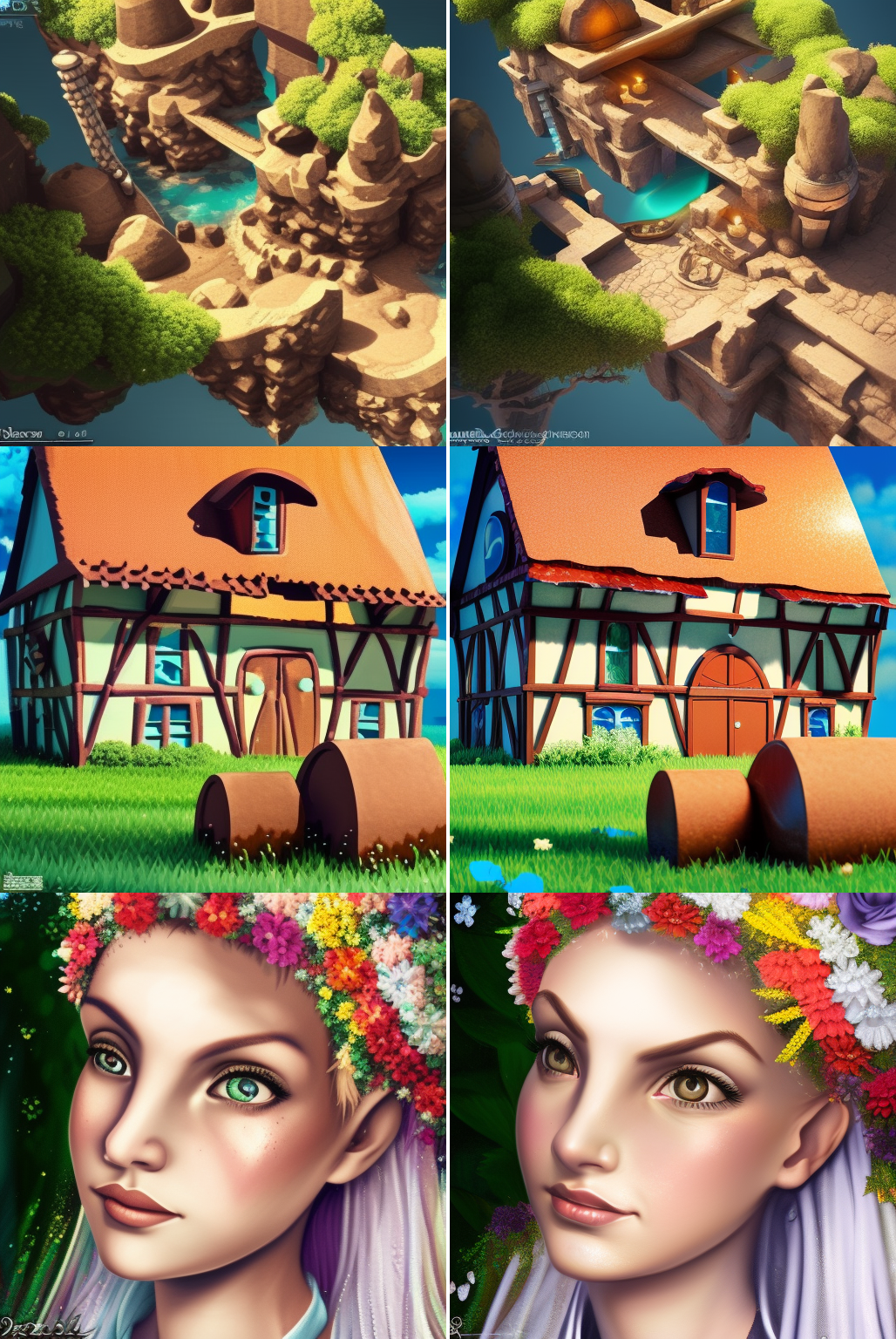}
    \hspace{\colgap}
    \methodfig{RI}{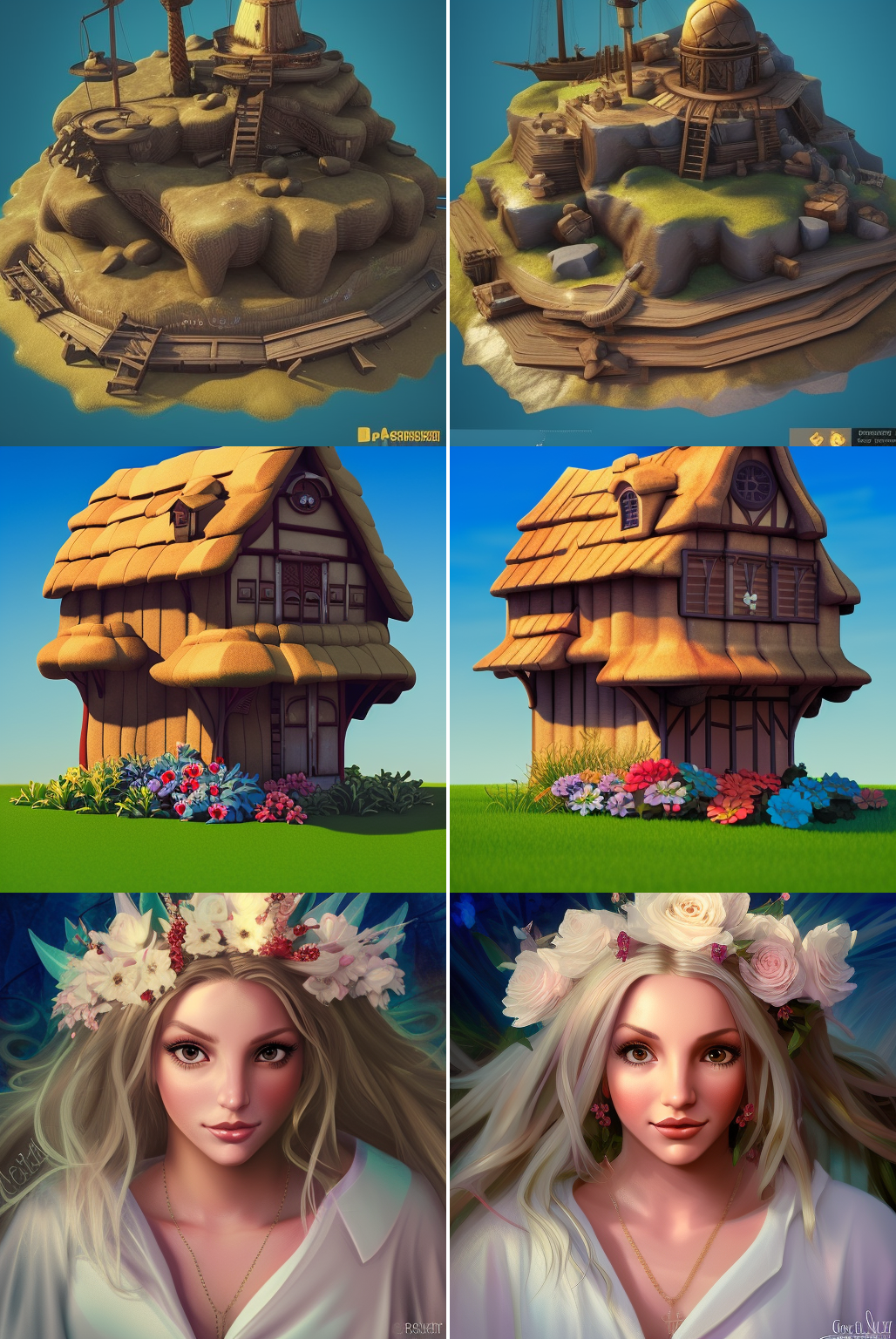}
    \hspace{\colgap}
    \methodfig{TR}{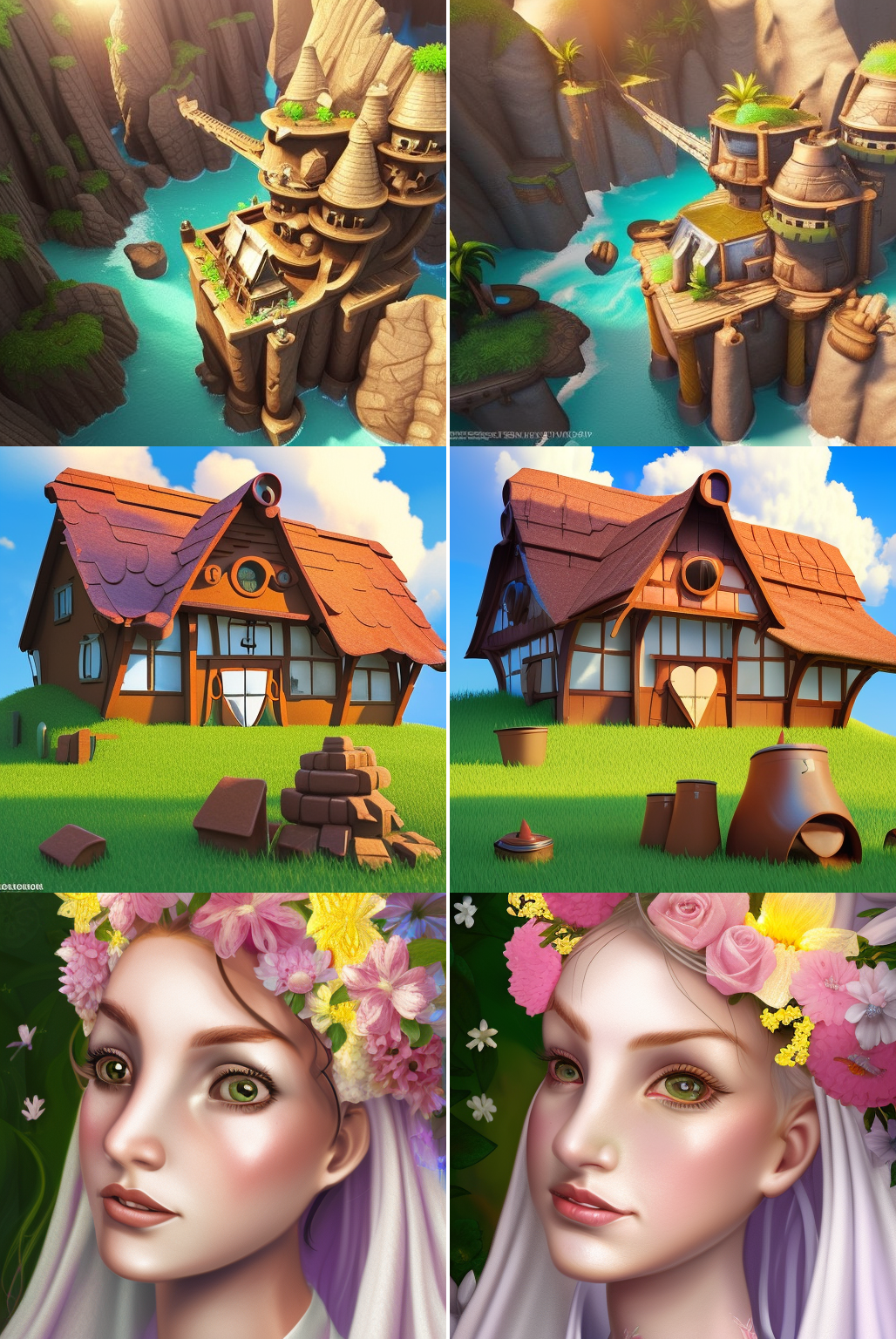}

    \vspace{\rowgap}

    \methodfig{SFW}{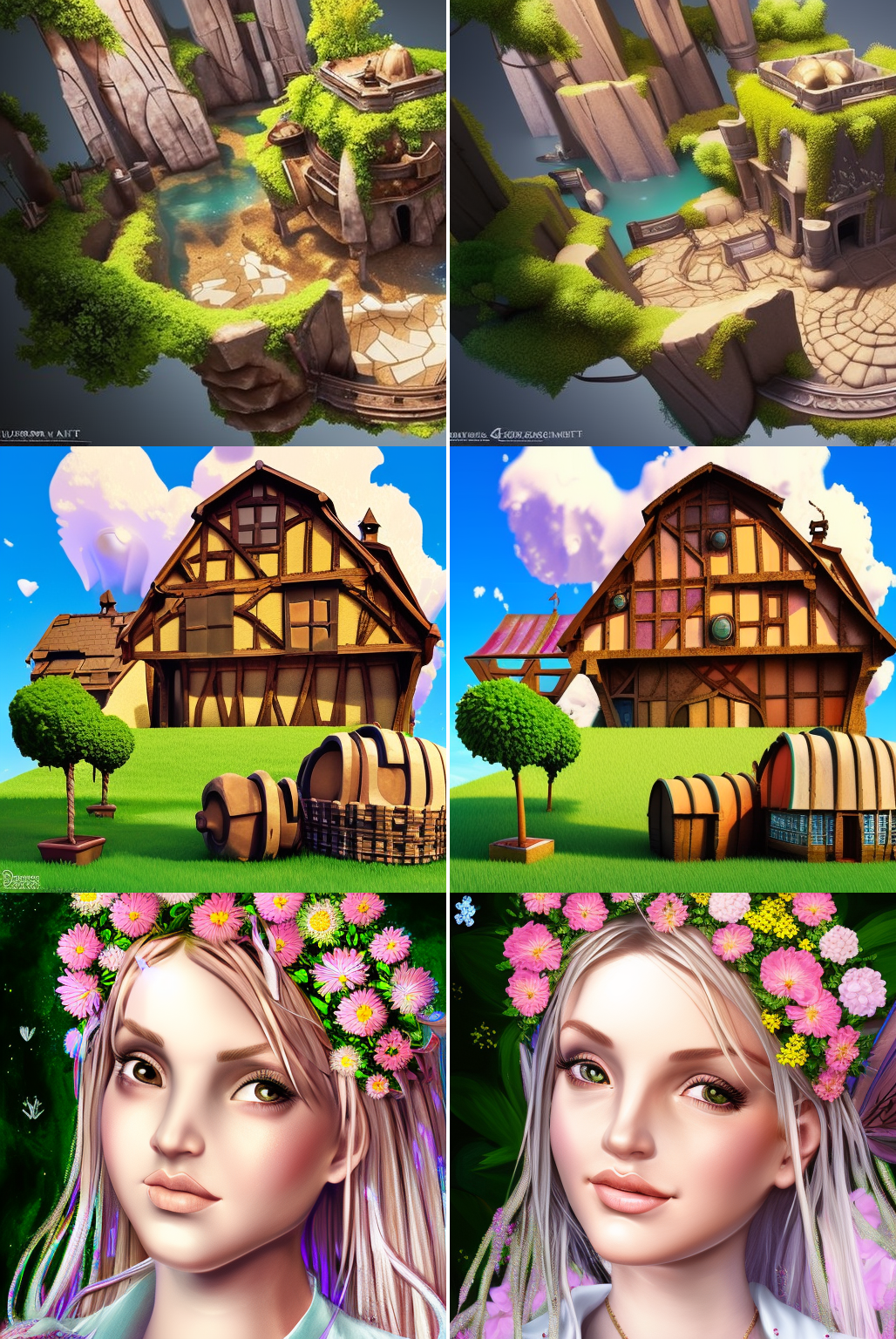}
    \hspace{\colgap}
    \methodfig{GS}{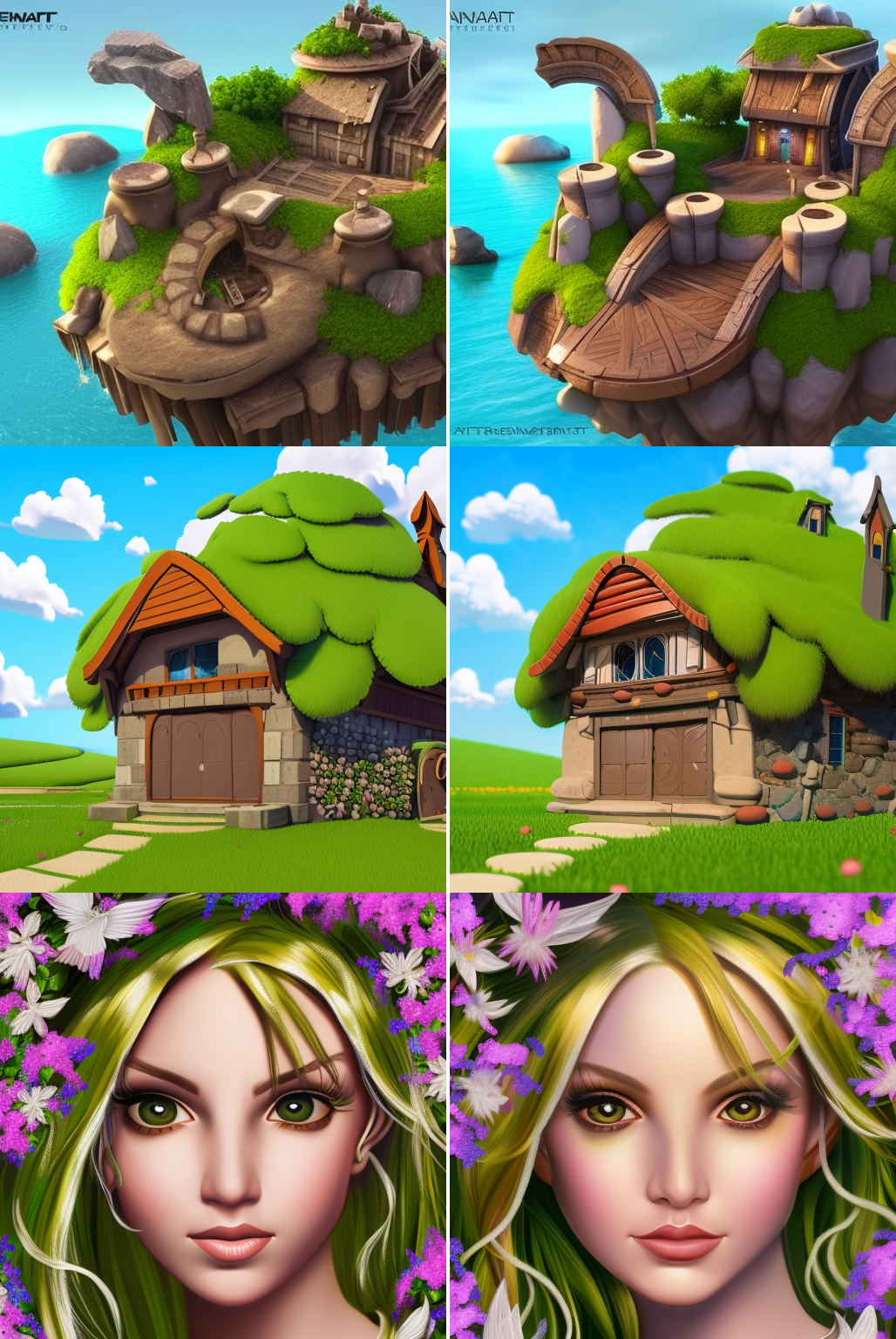}
    \hspace{\colgap}
    \methodfig{WIND}{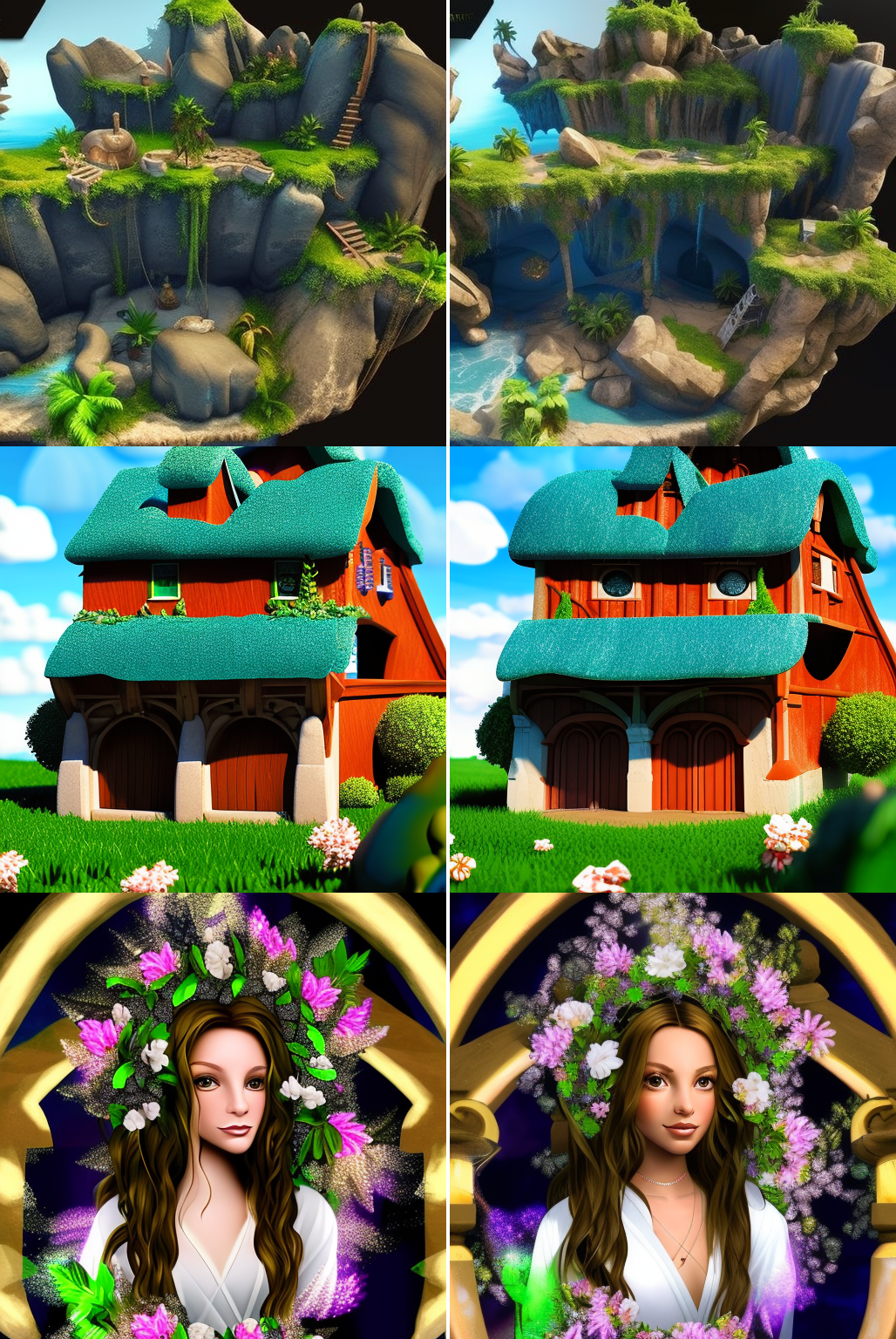}

    \caption{\textbf{Qualitative comparison of visual fidelity before 
    and after attack across nine watermarking methods.} In each subfigure, 
    the left column shows the watermarked images and the right column shows 
    the attacked images at attack strength $\lambda=0.50$.}
    \label{fig:qualitative_all_methods}
\end{figure*}

\clearpage

\subsection{Effect of Attack Strength on Image Fidelity}
\label{app:robustness_levels}

Figure~\ref{fig:robustness_compare} illustrates the fidelity--evasion 
trade-off across three representative watermarking methods of 
increasing robustness as $\lambda$ grows from $0.30$ to $0.60$. 
For weakly robust methods such as PRC, successful evasion is 
achieved at $\lambda=0.30$ with virtually no perceptual change 
relative to the original. For moderately robust methods such as 
ROBIN, minor variations in fine-grained texture appear at 
$\lambda=0.45$ but remain within an acceptable perceptual range. 
For the most robust method Tree-Ring, stronger perturbation at 
$\lambda=0.60$ is required, introducing slight deviation in 
high-frequency details while the overall scene structure and 
semantic content are well preserved. This progression is consistent 
with the robustness hierarchy reported in Table~1, and confirms 
that the minimum effective $\lambda$ is an intrinsic property of 
each watermarking scheme rather than a limitation of SHIFT.
\vspace{4mm}
\begin{figure*}[!htp]
    \centering
    \setlength{\tabcolsep}{2pt}
    \renewcommand{\arraystretch}{1.0}

    \begin{tabular}{c c c c c}
        \toprule
        \textbf{Method} & \textbf{Original} & 
        \textbf{$\lambda=0.30$} & 
        \textbf{$\lambda=0.45$} & 
        \textbf{$\lambda=0.60$} \\
        \midrule

        \textbf{PRC (Weak)} &
        \includegraphics[width=0.18\textwidth]
            {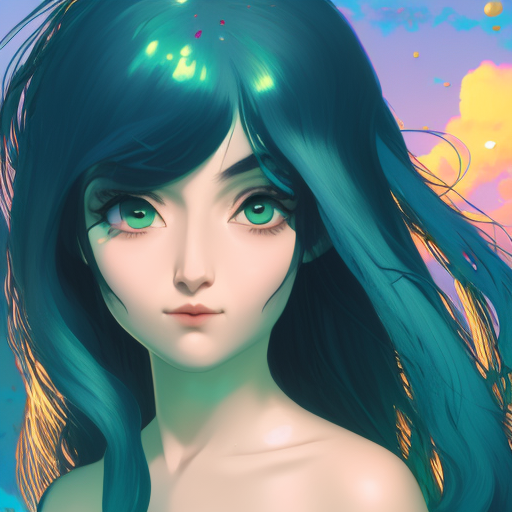} &
        \includegraphics[width=0.18\textwidth]
            {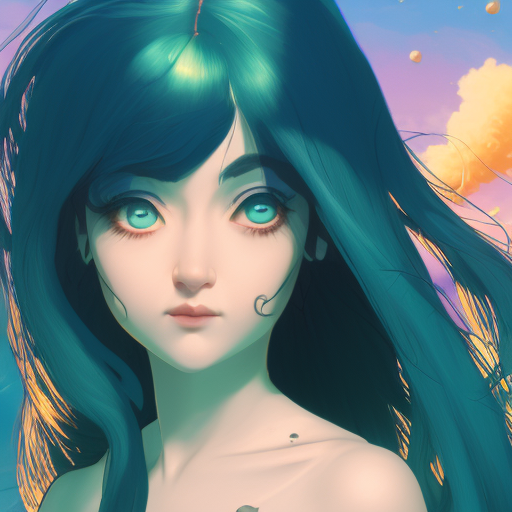} &
        \includegraphics[width=0.18\textwidth]
            {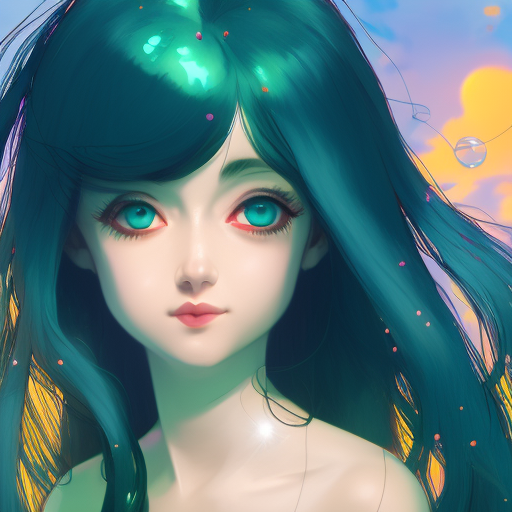} &
        \includegraphics[width=0.18\textwidth]
            {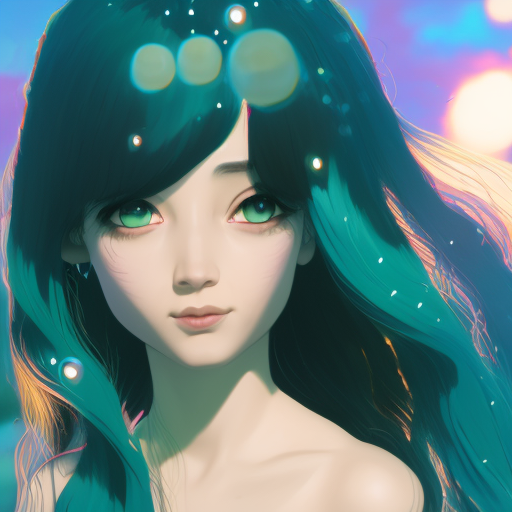} \\[4pt]

        \textbf{ROBIN (Moderate)} &
        \includegraphics[width=0.18\textwidth]
            {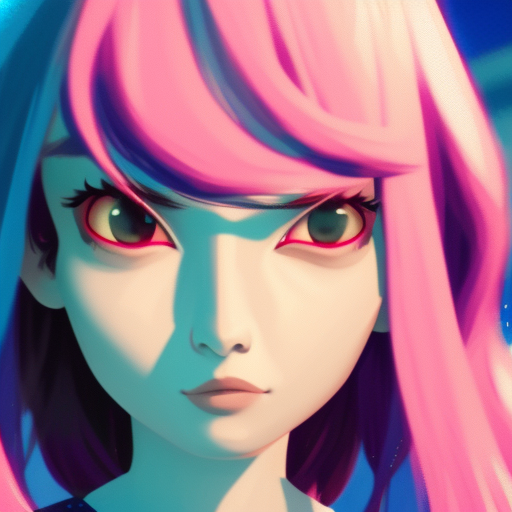} &
        \includegraphics[width=0.18\textwidth]
            {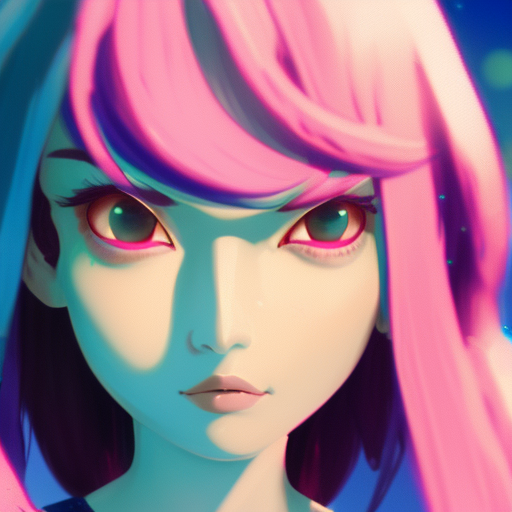} &
        \includegraphics[width=0.18\textwidth]
            {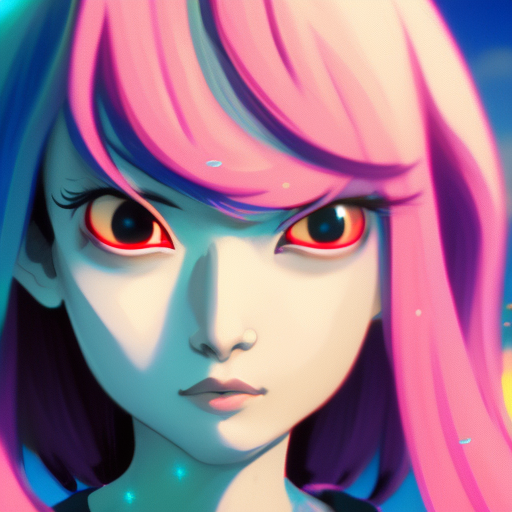} &
        \includegraphics[width=0.18\textwidth]
            {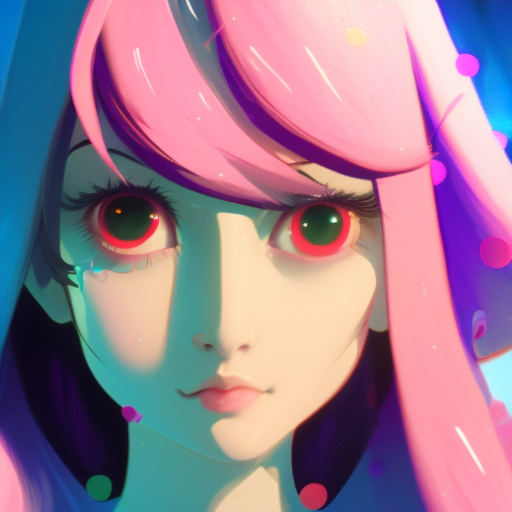} \\[4pt]

        \textbf{Tree-Ring (Strong)} &
        \includegraphics[width=0.18\textwidth]
            {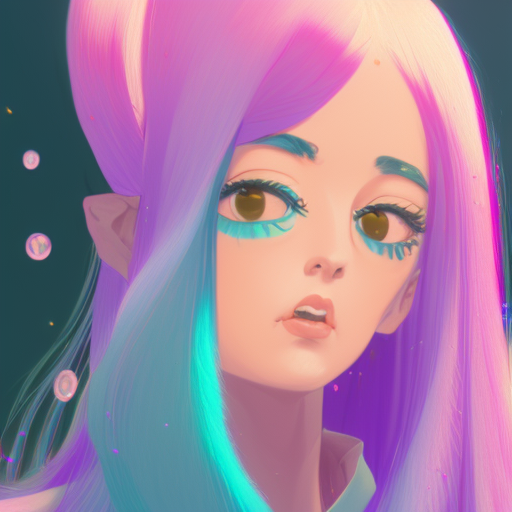} &
        \includegraphics[width=0.18\textwidth]
            {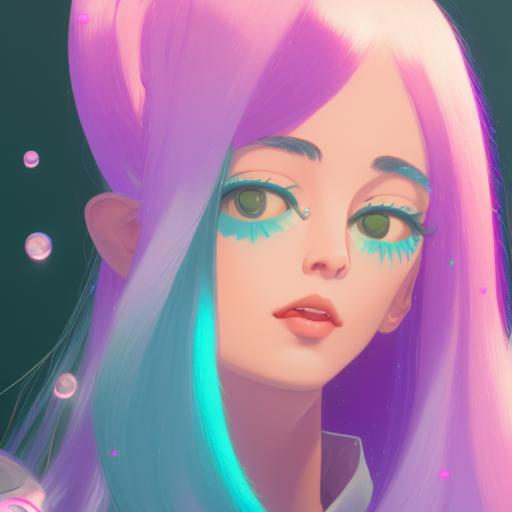} &
        \includegraphics[width=0.18\textwidth]
            {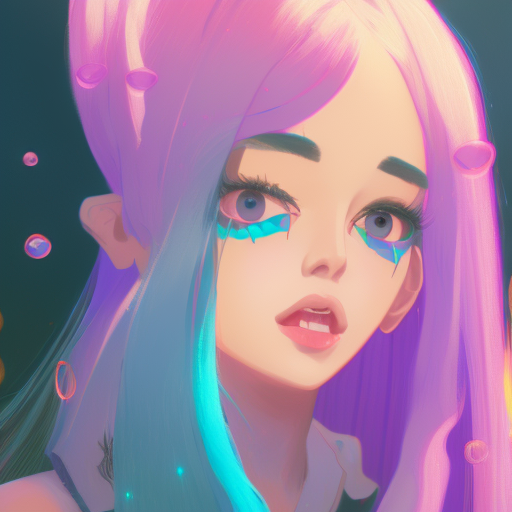} &
        \includegraphics[width=0.18\textwidth]
            {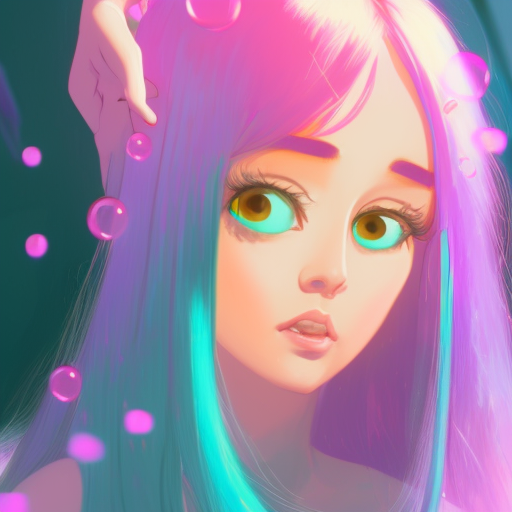} \\

        \bottomrule
    \end{tabular}

    \caption{\textbf{Visual comparison of watermark robustness under 
    different attack strengths.} PRC, ROBIN, and Tree-Ring represent 
    weakly, moderately, and highly robust watermarking methods, 
    respectively. As $\lambda$ increases, the attacked outputs reveal 
    clear differences in the fidelity--evasion trade-off across 
    robustness levels.}
    \label{fig:robustness_compare}
\end{figure*}

\clearpage
\subsection{Trajectory Decoupling: Noise Distance as a Function 
of Attack Strength}
\label{app:noise_distance}

Figure~\ref{fig:noise_distance} plots the mean $L_1$ and $L_2$ 
noise distances between the DDIM-inverted noise of the attacked 
image and the original watermark-carrying noise, as a function 
of $\lambda$, for all nine watermarking methods. Both metrics 
increase monotonically with $\lambda$ across every method, 
providing direct empirical evidence that stronger attacks 
progressively decouple the attacked sample from the original 
watermark trajectory in latent space. Notably, the curves 
converge to a narrow band of $L_1 \approx 1.05$--$1.07$ and 
$L_2 \approx 1.31$--$1.34$ at $\lambda=0.70$, regardless of 
the specific watermarking scheme. This method-agnostic saturation 
is consistent with the theoretical prediction of 
Theorem~\ref{prop:noise_distance_formal}: once sufficient 
forward diffusion depth is applied, the recovered noise 
distribution converges toward the random Gaussian baseline 
($L_2 \approx \sqrt{2d}$), independently of the embedded 
watermark structure.

\begin{figure*}[!htp]
    \centering
    \includegraphics[width=\textwidth]{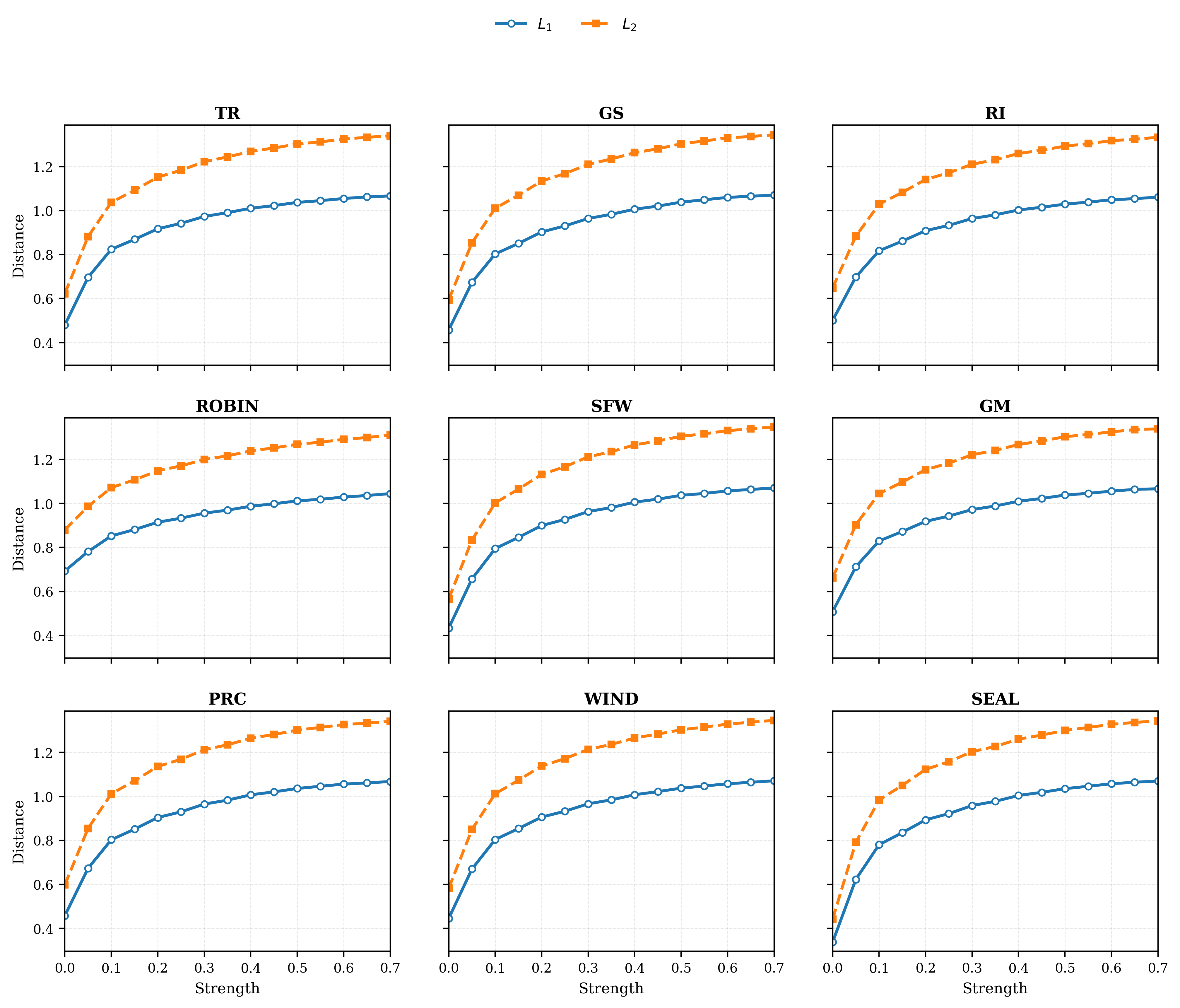}
    \caption{\textbf{Mean $L_1$ and $L_2$ noise distances as 
    functions of attack strength $\lambda$ across nine watermarking 
    methods.} Both metrics increase monotonically with $\lambda$, 
    confirming that SHIFT progressively decouples the attacked image 
    from the original watermark trajectory. The curves converge to 
    a method-agnostic regime near the random Gaussian baseline at 
    large $\lambda$, consistent with the theoretical bound in 
    Theorem~\ref{prop:noise_distance_formal}.}
    \label{fig:noise_distance}
\end{figure*}